\documentclass{article}
\usepackage[final]{neurips_2019}
\usepackage{hyperref}
\hypersetup{colorlinks=true,linkcolor=blue,citecolor=blue,urlcolor=blue}

\usepackage[utf8]{inputenc} 
\usepackage[T1]{fontenc}    
\usepackage{hyperref}       
\usepackage{url}            
\usepackage{booktabs}       
\usepackage{amsfonts}       
\usepackage{nicefrac}       
\usepackage{microtype}      
\usepackage{graphicx}
\usepackage{natbib}
\usepackage{wrapfig}

\usepackage{amsmath}
\usepackage{lipsum}
\usepackage{bbm}
\usepackage{caption}
\usepackage{subcaption}
\usepackage[utf8]{inputenc}
\usepackage[english]{babel}
\usepackage{amsthm}
\usepackage{algorithm,algorithmic}
\newcommand{\E}{\mathbb{E}}

\newcommand{\rrr}{\mathbb{R}}

\newcommand{\bx}{\mathbf{x}}

\newcommand{\by}{\mathbf{y}}

\newcommand{\ba}{\mathbf{a}}
\newcommand{\bb}{\mathbf{b}}

\newcommand{\bw}{\mathbf{w}}
\newcommand{\bA}{\mathbf{A}}
\newcommand{\bC}{\mathbf{C}}
\newcommand{\bI}{\mathbf{I}}
\newcommand{\bX}{\mathbf{X}}
\newcommand{\bY}{\mathbf{Y}}
\newcommand{\btheta}{\mathbf{\theta}}
\newcommand{\bSigma}{\mathbf{\Sigma}}
\newcommand{\bLambda}{\mathbf{\Lambda}}
\DeclareMathOperator{\Tr}{Tr}
\newcommand{\apx}{\textsc{Dash }}
\usepackage[font={footnotesize}]{caption}
\graphicspath{{images/}}

\newtheorem{theorem}{Theorem}
\newtheorem{corollary}[theorem]{Corollary}
\newtheorem{lemma}[theorem]{Lemma}

\newtheorem{remark}[theorem]{Remark}

\newtheorem{definition}[theorem]{Definition}
\title{Fast Parallel Algorithms for Statistical Subset Selection Problems}

\author{%
  Sharon Qian \\
  Harvard University\\
  \texttt{sharonqian@g.harvard.edu} \\
  \And
  Yaron Singer \\
  Harvard University\\
   \texttt{yaron@seas.harvard.edu} \\
}

\begin{document}
\maketitle

\begin{abstract}

In this paper, we propose a new framework for designing fast parallel algorithms for fundamental statistical subset selection tasks that include feature selection and experimental design.  Such tasks are known to be \emph{weakly submodular} and are amenable to optimization via the standard greedy algorithm.  Despite its desirable approximation guarantees, the greedy algorithm is inherently sequential and in the worst case, its parallel runtime is linear in the size of the data.
Recently, there has been a surge of interest in a parallel optimization technique called \emph{adaptive sampling} which produces solutions with desirable approximation guarantees for submodular maximization in exponentially faster parallel runtime.  Unfortunately, we show that for general weakly submodular functions such accelerations are impossible.  The major contribution in this paper is a novel relaxation of submodularity which we call \emph{differential submodularity}.  We first prove that differential submodularity characterizes objectives like feature selection and experimental design.  We then design an adaptive sampling algorithm for differentially submodular functions whose parallel runtime is logarithmic in the size of the data and achieves strong approximation guarantees.  Through experiments, we show the algorithm's performance is competitive with state-of-the-art methods and obtains dramatic speedups for feature selection and experimental design problems.       
\end{abstract}

\section{Introduction}
In fundamental statistics applications such as regression, classification and maximum likelihood estimation, we are often interested in selecting a subset of elements to optimize an objective function. In a series of recent works, both feature selection (selecting $k$ out of $n$ features) and  experimental design (choosing $k$ out of $n$ samples) were shown to be \emph{weakly submodular}~\cite{das2011,elenberg2018,bian2017}.  The notion of \emph{weak submodularity} was defined by Das and Kempe in~\cite{das2011} and quantifies the deviance of an objective function from submodularity.  Characterizations of weak submodularity are important as they allow proving guarantees of greedy algorithms in terms of the deviance of the objective function from submodularity.  More precisely, for objectives that are $\gamma$-weakly submodular (for $\gamma$ that depends on the objective, see preliminaries Section~\ref{sec:prelim}), the greedy algorithm is shown to return a $1-1/e^{\gamma}$ approximation to the optimal subset.  

\paragraph{Greedy is sequential and cannot be parallelized.} For large data sets where one wishes to take advantage of parallelization, greedy algorithms are impractical.  Greedy algorithms for feature selection such as forward stepwise regression iteratively add the feature with the largest marginal contribution to the objective which requires computing the contribution of each feature in every iteration. Thus, the parallel runtime of the forward stepwise algorithm and greedy algorithms in general, scale linearly with the number of features we want to select. In cases where the computation of the objective function across all elements is expensive or the dataset is large, this can be computationally infeasible. 

\paragraph{Adaptive sampling for fast parallel submodular maximization.}  In a recent line of work initiated by \cite{BS18a}, adaptive sampling techniques have been used for maximizing submodular functions under varying constraints \cite{balkanski2018, chekuri2019, chekuri2018matroid, balkanski2018non, chen2018, ene2018, farbach2019, balkanski2018matroid, ene2019, FMZ19}. Intuitively, instead of growing the solution set element-wise, adaptive sampling adds a large set of elements to the solution at each round which allows the algorithm to be highly parallelizable. In particular, for canonical submodular maximization problems, one can obtain approximation guarantees arbitrarily close to the one obtained by greedy (which is optimal for polynomial time algorithms~\cite{nemhauser1978best}) in \emph{exponentially} faster parallel runtime.

\paragraph{In general, adaptive sampling fails for weakly submodular functions.}   Adaptive sampling techniques add large sets of  high valued elements in each round by filtering elements with low marginal contributions.  This enables these algorithms to terminate in a small number of rounds.  For weak submodularity, this approach renders arbitrarily poor approximations. In Appendix \ref{appendix:toy1}, we use an example of a weakly submodular function from \cite{elenberg2017} where adaptive sampling techniques have an arbitrarily poor approximation guarantee. Thus, if we wish to utilize adaptive sampling to parallelize algorithms for applications such as feature selection and experimental design, we need a stronger characterization of these objectives which is amenable to parallelization.

\subsection{Differential Submodularity}
In this paper, we introduce an alternative measure to quantify the deviation from submodularity which we call \emph{differential submodularity}, defined below.  We use $f_{S}(A)$ to denote $f(S\cup A) - f(S)$.  

\begin{definition}
A function $f: 2^N \rightarrow \rrr_+$ is $\alpha$-differentially submodular for $\alpha\in[0,1]$, if there exist two submodular functions  $h, g$  s.t. for any $S, A \subseteq N$, we have that $g_S(A) \geq \alpha \cdot h_S(A)$ and
$$ g_S(A) \leq f_S(A) \leq h_S(A)$$
\end{definition}

A 1-differentially submodular function is submodular and a 0-differentially submodular function can be arbitrarily far from submodularity. In Figure \ref{fig:toy}, we show a depiction of differential submodularity (blue lines) calculated from the feature selection objective by fixing an element $a$ and randomly sampling sets $S$ of size 100 to compute the marginal contribution $f_S(a)$ on a real dataset. For a differentially submodular function (blue lines), the property of decreasing marginal contributions does not hold but can be bounded by two submodular functions (red) with such property.
\begin{wrapfigure}[10]{r}{0.35\textwidth}
\centering
\vspace{-1.2em}
\includegraphics[width=0.36\textwidth]{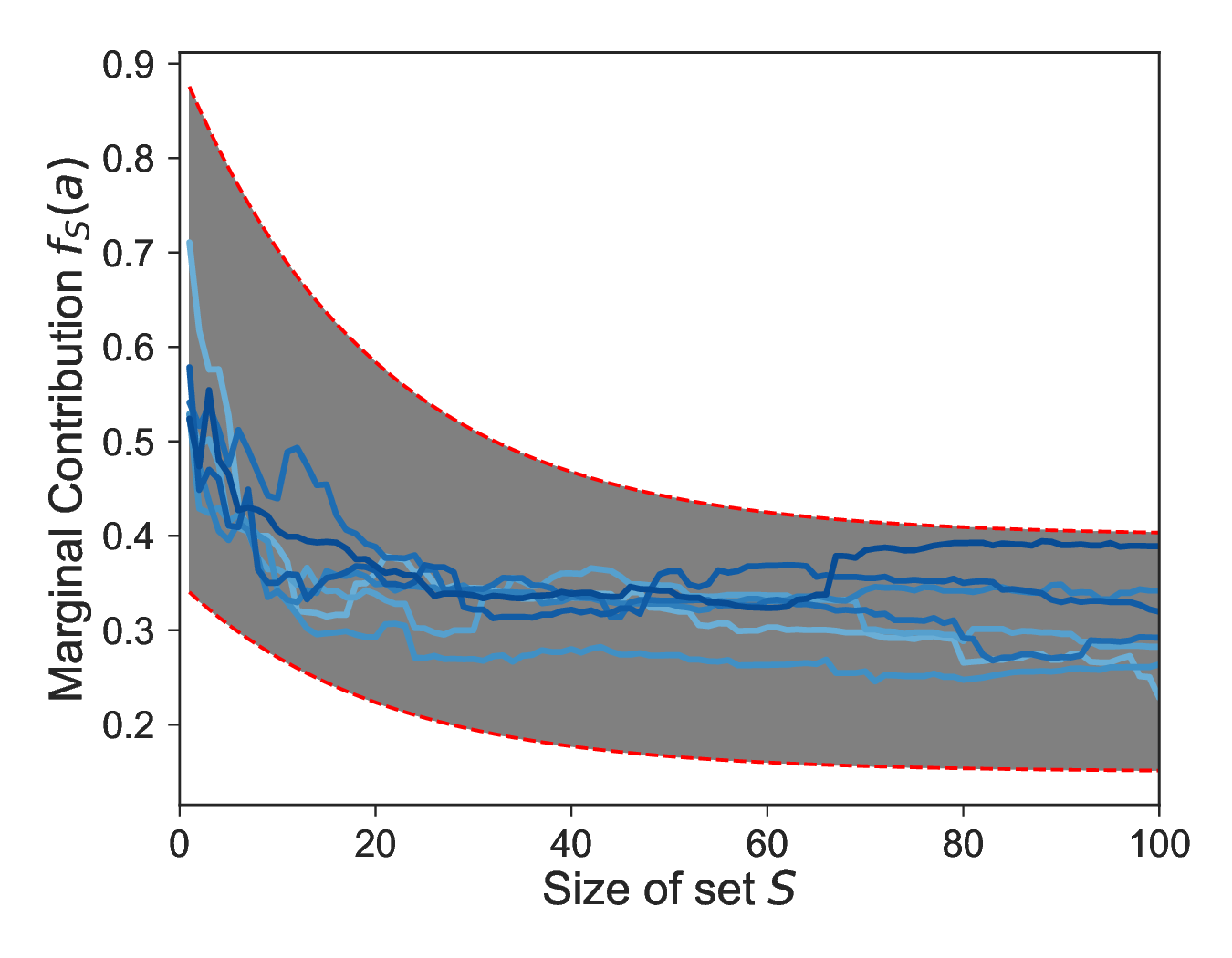}
\vspace{-1.8em}
\caption{Marginal contribution of differentially submodular function.}
\label{fig:toy}
\end{wrapfigure}

As we prove in this paper, applications such as feature selection for regression and classification as well as experimental design are all $\gamma^2$-differentially submodular, where $\gamma$ corresponds to their weak submodularity ratios~\cite{elenberg2018,bian2017}.  The power of this characterization is that it allows for parallelization with strong approximation guarantees. We do this by designing an adaptive sampling algorithm that leverages the differential submodularity structure and has bounded approximation guarantees in terms of the differential submodularity ratios.  

\subsection{Main results} 
Our main result is that for objectives such as feature selection for regression and classification and Bayesian A-optimality experimental design which are all $\gamma$-weakly submodular, there is an approximation guarantee arbitrarily close to $1-1/e^{\gamma^4}$ for maximization under cardinality constraints in $\mathcal O(\log n)$ adaptive rounds (see adaptivity definition in Section \ref{sec:prelim}). Thus, while the approximation is inferior to the $1- 1/e^{\gamma }$ obtained by greedy, our algorithm has exponentially fewer rounds. Importantly, using experiments we show that empirically it has comparable terminal values to the greedy algorithm, greatly outperforms its theoretical lower bound, and obtains the result with two to eight-fold speedups. We achieve our result by proving these objectives are $\alpha$-differentially submodular and designing an adaptive sampling algorithm that gives a $1-1/e^{\alpha^2}$ approximation for maximizing any $\alpha$-differentially submodular function under a cardinality constraint. 

\paragraph{Conceptual overview.} For the past decade, fundamental problems in machine learning have been analyzed through relaxed notions of submodularity (See details on different relaxations of submodularity and relationship to differential submodularity in Appendix \ref{appendix:notions}). Our main conceptual contribution is the framework of differential submodularity which is purposefully designed to enable fast parallelization techniques that previously-studied relaxations of submodularity do not. Specifically, although stronger than weak submodularity, we can prove direct relationships between objectives' weak submodularity ratios and their differential submodularity ratios which allows getting strong approximations and exponentially faster parallel runtime. We note that differential submodularity is also applicable to more recent parallel optimization techniques such as adaptive sequencing~\cite{balkanski2018matroid}.

\paragraph {Technical overview.} From a purely technical perspective, there are two major challenges addressed in this work.  The first pertains to the characterization of the objectives in terms of differential submodularity and the second is the design of an adaptive sampling algorithm for differentially submodular functions.  Previous adaptive sampling algorithms are purposefully designed for submodular functions and cannot be applied when the objective function is not submodular (example in Appendix \ref{appendix:toy2}). In these cases, the marginal contribution of individual elements is not necessarily subadditive to the marginal contribution of the set of elements combined. Thus, the standard analysis of adaptive sampling, where we attempt to add large sets of elements to the solution set by assessing the value of individual elements, does not hold. By leveraging the fact that marginal contributions of differentially submodular functions can be bounded by marginal contributions of submodular functions, we can approximate the marginal contribution of a set by assessing the marginal contribution of its elements. This framework allows us to leverage parallelizable algorithms to show a stronger approximation guarantee in exponentially fewer rounds.

\paragraph {Paper organization.} \ We first introduce preliminary definitions in Section \ref{sec:prelim} followed by introducing our main framework of differential submodularity and its reduction to feature selection and experimental design objectives in Section \ref{sec:obj}. We then introduce an algorithm for selection problems using adaptive sampling in Section \ref{sec:algo} and conclude with experiments in Section \ref{sec:exp}. Due to space constraints, most proofs of the analysis are deferred to the Appendix.

\section{Preliminaries} \label{sec:prelim}

For a positive integer $n$, we use $[n]$ to denote the set $\{1,2,\ldots,n\}$. Boldface lower and upper case letters denote vectors and matrices respectively: $\ba, \bx, \by$ represent vectors and $\bA,\bX, \bY$ represent matrices. Unbolded lower and upper case letters present elements and sets respectively: $a,x,y$ represent elements and $A,X,Y$ represent sets. 
For a matrix $\bX \in \rrr^{d\times n}$ and $S\subseteq [n]$, we denote submatrices by column indices by $\bX_S$. For vectors, we use $\bx_S$ to denote supports $\text{supp}(\bx) \subseteq S$. To connect the discrete function $f(S)$ to a continuous function, we let $f(S) = \ell(\bw^{(S)})$, where $\bw^{(S)}$ denotes the $\bw$ that maximizes $\ell(\cdot)$ subject to $\text{supp}(\bw) \subseteq S$.

\paragraph{Submodularity and weak submodularity.} A function $f : 2^N \rightarrow \rrr_+$ is submodular if $f_S(a) \geq f_T (a)$ for all $a\in N\backslash T$ and $S\subseteq T \subseteq N$. It is {\it monotone} if $f(S)\leq f(T)$ for all $S\subseteq T$. We assume that $f$ is normalized and non-negative, i.e., $0 \leq f(S) \leq 1$ for all $S \subseteq N$, and monotone.  
The concept of \emph{weak submodularity} is a relaxation of submodularity, defined via the {\it submodularity ratio}:

\begin{definition} \label{def:ratio}~\cite{das2011}
The \textbf{submodularity ratio} of $f: 2^N \rightarrow \rrr_+$ is defined as, for all $A\subseteq N$, 
$$\gamma_k = \min_{A\subseteq N, S: |A|\leq k} \frac{\sum_{a\in A} f_S(a)}{f_S(A)}.$$
\end{definition}

Functions with submodularity ratios $\gamma = \min_{k} \gamma_k < 1$ are $\gamma$-{\it weakly submodular}. 

\paragraph{Adaptivity.}
The {\it adaptivity} of algorithms refers to the number of sequential rounds of queries it makes when polynomially-many  queries can be executed in parallel in each round. 

\begin{definition}
For a function $f$, an algorithm is $r$-adaptive if every query $f(S)$ given a set $S$ occurs at a round $i \in [r]$ such that $S$ is independent of the values $f(S')$ of all other queries at round $i$.
\end{definition}

Adaptivity is an information theoretic measure of parallel-runtime that can be translated to standard parallel computation frameworks such as PRAM (See Appendix \ref{app:pram}). Therefore, like all previous work on adaptivity on submodular maximization, we are interested in algorithms that have low adaptivity since they are parallelizable and scalable for large datasets \cite{balkanski2019, balkanski2018, chekuri2019, chekuri2018matroid, balkanski2018non, chen2018, ene2018, farbach2019, balkanski2018matroid, ene2019,FMZ19}.


\section{Feature Selection and A-Optimal Design are Differentially Submodular }\label{sec:obj}
We begin by characterizing differential submodularity in terms of \emph{restricted strong concavity} and \emph{restricted smoothness} defined as follows.
\begin{definition}
\cite{elenberg2018}
Let $\Omega$ be a subset of $\rrr^n \times \rrr^n$ and $\ell : \rrr^n \rightarrow \rrr$ be a continuously differentiable function. A function $\ell$ is \textbf{restricted strong concave} (RSC) with parameter $m_\Omega$ and \textbf{restricted smooth} (RSM) with parameter $M_\Omega$ if, for all $(\by, \bx) \in \Omega$,
\begin{eqnarray}
- \frac{m_\Omega}{2} \| \by - \bx \|^2_2 &\geq& \ell(\by) - \ell(\bx) - \langle \nabla \ell(\bx), \by - \bx\rangle 
\geq -\frac{M_\Omega}{2} \| \by - \bx \|^2_2 \nonumber 
\end{eqnarray}

\end{definition}

Before connecting our notion of differential submodularity to RSC/RSM properties, we first define concavity and smoothness parameters on subsets of $\Omega$. If $\Omega' \subseteq \Omega$, then $M_{\Omega'} \leq M_\Omega$ and $m_{\Omega'} \geq m_{\Omega}$. 
\begin{definition}\label{def:rsc}
We define the domain of $s$-sparse vectors as $\Omega_{s} = \{(\bx, \by) : \|\bx\|_0 \leq s, \|\by\|_0 \leq s, \|\bx - \by\|_0 \leq s\}$. If $ t\geq s$, $M_s \leq M_t$ and $m_s\geq m_t$.
\end{definition}

\begin{theorem}\label{thm:param}
Suppose $\ell(\cdot)$ is RSC/RSM on $s$-sparse subdomains $\Omega_s$ with parameters $m_s,M_s$ for $s \leq 2k$. Then, for $t = |S| + k, s = |S| + 1$, the objective $f(S) = \ell(\bw^{(S)})$  is differentially submodular s.t. for $S,A\subseteq N$, $|A| \leq k$,
$ \frac{m_{s}}{M_{t}} \tilde f_S(A) \leq f_S(A) \leq \frac{M_{s}}{m_{t}} \tilde f_S(A)$,
where $\tilde f_S(A) = \sum_{a\in A} f_S(a)$.
\end{theorem}

\begin{proof}
We first prove the lower bound of the inequality. We define $\bx_{(S \cup A)} = \frac{1}{M_{t}} \nabla \ell (\bw^{(S)})_A + \bw^{(S)}$ and  use the strong concavity of $\ell(\cdot)$ to lower bound  $f_S(A)$:
\begin{eqnarray}
f_S(A) 
\geq  \ell(\bx_{(S\cup A)}) - \ell(\bw^{(S)}) 
&\geq& \langle \nabla \ell (\bw ^{(S)}), \bx_{(S\cup A)} - \bw^{(S)} \rangle 
- \frac{M_{t}}{2}\| \bx_{(S\cup A)} - \bw^{(S)} \|^2_2
\nonumber 
\\
&\geq& \frac{1}{2M_{t}} \|  \nabla \ell (\bw ^{(S)})_A \|^2_2 \label{eqn:lower}
\end{eqnarray}
where the first inequality follows from the optimality of $\ell(\bw^{(S\cup A)})$ for vectors with support $S\cup A$ and the last inequality is by the definition of $\bx_{(S \cup A)}$.

We also can use smoothness of $\ell(\cdot)$ to upper bound the marginal contribution of each element in $A$ to $S$, $f_S(a)$. We define $\bx_{(S \cup a)} = \frac{1}{m_{ s}} \nabla \ell (\bw^{(S)})_a + \bw^{(S)}$. For $a\in A$,
\begin{eqnarray}
f_S(a) = \ell(\bw^{(S\cup a)}) - \ell(\bw^{(S)}) 
&\leq&  \langle \nabla \ell (\bw ^{(S)}), \bx_{(S \cup a)} - \bw^{(S)} \rangle  
- \frac{m_{ s}}{2}\| \bx_{(S \cup a)} - \bw^{(S)} \|^2_2 \nonumber \\
&\leq& \frac{1}{2m_{ s}} \|  \nabla \ell (\bw ^{(S)})_a \|^2_2 
\end{eqnarray}
where the last inequality follows from the definition of $\bx_{(S \cup a)}$.
Summing across all $a\in A$, we get 
\begin{eqnarray}
\sum_{a\in A} f_S(a) &\leq& \sum_{a\in A} \frac{1}{2m_{ s}} \|  \nabla \ell (\bw ^{(S)})_a \|^2_2 
= \frac{1}{2m_{s}} \|  \nabla \ell (\bw ^{(S)})_A \|^2_2 \label{eqn:upper}
\end{eqnarray}

By combining (\ref{eqn:lower}) and (\ref{eqn:upper}), we can get the desired lower bound of $f_S(A)$.
To get the upper bound on the marginals, we can use the lower bound of submodularity ratio $\gamma_{S, k}$ of $f$ from Elenberg et al. \cite{elenberg2018}, which is no less than $\frac{m_{t}}{M_{s}}$. Then, by letting $\tilde f_S(A) = \sum_{a\in A} f_S(a)$, we can complete the proof and show that the marginals can be bounded.
\end{proof}

We can further generalize the previous lemma to all sets $S,A\subseteq N$, by using the general RSC/RSM parameters $m, M$ associated with $\Omega_n$, where $n \geq t,s$. From Definition \ref{def:rsc}, since $\Omega_{s} \subseteq \Omega_{t} \subseteq \Omega_n$, $M_{s} \leq M_{t} \leq M$ and $m_{s} \geq m_{t} \geq m$. Thus, we can weaken the bounds from Lemma \ref{thm:param} to get $\frac{m}{M} \tilde f_S(A) \leq f_S(A) \leq \frac{M}{m} \tilde f_S(A)$ which is a $\gamma^2$-differentially submodular function for $\gamma = \frac{m}{M}$.

\subsection{Differential submodularity bounds for statistical subset selection problems}
We now connect differential submodularity to feature selection and experimental design objectives.  We also show that even when adding diversity-promoting terms $d(S)$ as in \cite{das2012} the functions remain differentially submodular. Due to space limitations, proofs are deferred to Appendix \ref{appendix:main2}.
 
 {\bf Feature selection for regression.} \ For a response variable $\by \in \rrr^d$ and feature matrix $\bX \in \rrr^{d\times n}$, the objective is the maximization of the $\ell_2$-utility function that represents the variance reduction of $\by$ given the feature set $S$: 
\begin{eqnarray}
\ell_{\texttt{reg}}(\by, \bw^{(S)}) = \|\by\|_2^2 - \| \by-\bX_S\bw\|^2_2 \nonumber
\end{eqnarray}

We can bound the marginals by eigenvalues of the feature covariance matrix. We denote the minimum and maximum eigenvalues of the $k$-sparse feature covariance matrix by $\lambda_{min}(k)$ and $\lambda_{max}(k)$.

\begin{corollary} \label{corollary:lin}
Let $\gamma = \frac{\lambda_{min}(2k)}{\lambda_{max}(2k)}$ and $d : 2^N \rightarrow \rrr_+$ be a submodular diversity function.  Then $f(S) = \ell_{\texttt{reg}}(\bw^{(S)})$ and $f_{\texttt{div}}(S) = \ell_{\texttt{reg}}(\bw^{(S)}) + d(S)$ are $\gamma^2$-differentially submodular.
\end{corollary}

We note that \cite{das2011} use a different objective function to measure the goodness of fit $R^2$. In Appendix~\ref{appendix:eig}, we show an analogous bound for the objective used in~\cite{das2011}. Our lower bound is consistent with the result in Lemma 2.4 from Das and Kempe \cite{das2011}. 

\paragraph{Feature selection for classification.} For classification, we wish to select the best $k$ columns from $\bX \in \rrr^{d\times n}$ to predict a categorical variable $\by \in \rrr^d$. We use the following log-likelihood objective in logistic regression to select features. For a categorical variable $\by \in \rrr^d$, the objective in selecting the elements to form a solution set is the maximization of the log-likelihood function for a given $S$: 
\begin{eqnarray}
\ell_{\texttt{class}}(\by, \bw^{(S)}) = \sum_{i=1}^d y_i (\bX_S \bw) - \log(1+e^{ \bX_S \bw}) \nonumber
\end{eqnarray}

We denote $m$ and $M$ to be the RSC/RSM parameters on the feature matrix $\bX$.  For $\gamma=\frac{m}{M}$~\cite{elenberg2018} show that the feature selection objective for classification is $\gamma$-weakly submodular. 

\begin{corollary}  \label{corollary:log}
Let $\gamma = \frac{m}{M}$ and $d : 2^N \rightarrow \rrr_+$ be a submodular diversity function.  Then $f(S) = \ell_{\texttt{class}}(\bw^{(S)})$ and $f_{\texttt{div}}(S) =  \ell_{\texttt{class}}(\bw^{(S)}) + d(S)$ are $\gamma^2$-differentially submodular.
\end{corollary}

\paragraph{Bayesian A-optimality for experimental design.} In experimental design, we wish to select the set of experimental samples $\bx_i$ from $\bX \in \rrr^{d\times n}$ to maximally reduce variance in the parameter posterior distribution. We now show that the objective for selecting diverse experiments using Bayesian A-optimality criterion is differentially submodular. We denote $\bLambda = \beta^2\bI$ as the prior that takes the form of an isotropic Gaussian and $\sigma^2$ as variance (See Appendix \ref{app:bayes} for more details).
 
\begin{corollary} \label{corollary:bayes}
Let $\gamma = \frac{\beta^2}{\|\bX\|^2(\beta^2+\sigma^{-2} \|\bX\|^2)}$ and $d : 2^N \rightarrow \rrr_+$ be a submodular diversity function, then the objectives of Bayesian A-optimality defined by $f_{\texttt{A-opt}}(S) = \Tr(\bLambda^{-1}) - \Tr((\bLambda + \sigma^{-2}\bX_S \bX_S^T)^{-1})$ and the diverse analog defined by $f_{\texttt{A-div}}(S) = f_{\texttt{A-opt}}(S) + d(S)$ are $\gamma^2$-differentially submodular.
\end{corollary}

\section{The Algorithm} \label{sec:algo}
We now present the \apx (\textsc{Differentially-Adaptive-SHampling}) algorithm for maximizing differentially submodular objectives with logarithmic adaptivity. 
Similar to recent works on low adaptivity algorithms \cite{balkanski2019, balkanski2018, chekuri2019, chekuri2018matroid, balkanski2018non, chen2018, ene2018, farbach2019, balkanski2018matroid, ene2019}, this algorithm is a variant of the adaptive sampling technique introduced in~\cite{BS18a}.  The adaptive sampling algorithm for submodular functions, where $\alpha=1$, is not guaranteed to terminate for non-submodular functions (See Appendix \ref{appendix:toy2}). Thus, we design a variant to specifically address differential submodularity to parallelize the maximization of non-submodular objectives.

\begin{algorithm}
\caption{\apx$(N, r, \alpha)$}
    \begin{algorithmic}[1]
    \STATE {\bf Input} Ground set $N$, number of outer-iterations $r$, 
    differential submodularity parameter $\alpha$
      \normalsize
      \STATE{$S \leftarrow \emptyset$, $X \leftarrow N$}
	\FOR{$r$ iterations}
      \STATE $t := (1-\epsilon) (f(O) - f(S))$
      
      \WHILE{$\mathbb E_{R\sim \mathcal U(X)} [f_S(R)] < \alpha^2 \frac{t}{r}$}
      	\STATE {$X \leftarrow X \backslash \{ a : \mathbb E_{R\sim \mathcal U(X)}[ f_{S \cup (R\backslash \{a\})} (a)] < \alpha(1+\frac{\epsilon}{2}) t/k
	\}$}
      \ENDWHILE
            \STATE{$S \leftarrow S \cup R$ where $R \sim \mathcal U(X)$}
     	\ENDFOR
       \STATE {\bf return} $S$
    \end{algorithmic}
 \end{algorithm}
 
\paragraph {Algorithm overview.} At each round, the \apx algorithm selects good elements determined by their individual marginal contributions and attempts to add a set of $k/r$ elements to the solution set $S$. The decision to label elements as "good" or "bad" depends on the threshold $t$ which quantifies the distance between the elements that have been selected and \texttt{OPT}. This elimination step takes place in the \texttt{while} loop and effectively filters out elements with low marginal contributions. The algorithm terminates when $k$ elements have been selected or when the value of $f(S)$ is sufficiently close to \texttt{OPT}.

The algorithm presented is an idealized version because we cannot exactly calculate expectations, and \texttt{OPT} and differential submodularity parameter $\alpha$ are unknown. We can estimate the expectations by increasing sampling of the oracle and we can guess \texttt{OPT} and $\alpha$ through parallelizing multiple guesses (See Appendix \ref{appendix:queries} for more details).

\paragraph {Algorithm analysis.}  We now outline the proof sketch of the approximation guarantee of $f(S)$ using \textsc{Dash}.  In our analysis, we denote the optimal solution as $\texttt{OPT} = f(O)$ where $O = \text{argmax}_{|S| \leq k} f(S)$ and $k$ is a cardinality constraint parameter. Proof details can be found in Appendix \ref{appendix:main}.

\begin{theorem} \label{thm:apx}
Let $f$ be a monotone, $\alpha$-differentially submodular function where $\alpha \in [0, 1]$, then, for any $\epsilon > 0$, \apx is a $\log_{1+\epsilon/2}(n)$ adaptive algorithm that obtains the following approximation for the set $S$ that is returned by the algorithm
$$f(S) \geq (1-1/e^{\alpha^2} - \epsilon) f(O).$$
\end{theorem}

The key adaptation for $\alpha$-differential submodular functions appears in the thresholds of the algorithm, one to filter out elements and another to lower bound the marginal contribution of the set added in each round. The additional $\alpha$ factor in the \texttt{while} condition compared to the single element marginal contribution threshold is a result of differential submodularity properties and guarantees termination. 

To prove the theorem, we lower bound the marginal contribution of selected elements $X_\rho$ at each iteration $\rho$: $f_S(X_\rho) \geq \frac{\alpha^2}{r}(1-\epsilon) (f(O) - f(S))$ (Lemma \ref{lemma:apx} in Appendix \ref{appendix:lemma}).

We can show that the algorithm terminates in $\log_{1+\epsilon/2}(n)$ rounds  (Lemma \ref{lemma:log} in Appendix \ref{appendix:lemma}). Then, using the lower bound of the marginal contribution of a set at each round $f_S(X_\rho)$ in conjunction with an inductive proof, we get the desired result. 

We have seen in Corollary \ref{corollary:lin}, \ref{corollary:log} and \ref{corollary:bayes} that the feature selection and Bayesian experimental design problems are differentially submodular. Thus, we can apply \apx to these problems to obtain the $f(S) \geq (1-1/e^{\alpha^2} - \epsilon) f(O)$ guarantee from Theorem \ref{thm:apx}.

\begin{figure*}[h]
\vskip 0.2in
\begin{center}
\begin{minipage}{0.28\textwidth}
\includegraphics[width=\textwidth]{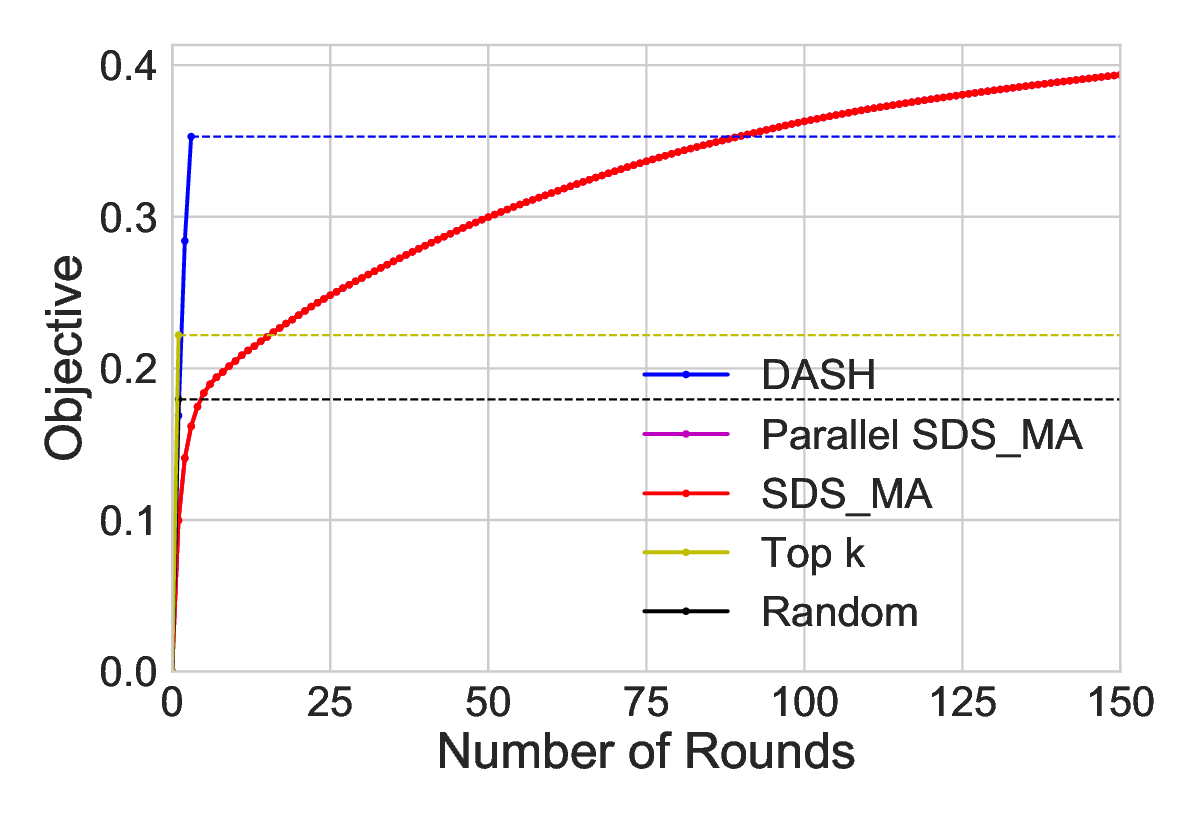}
\subcaption{} \label{fig:lin_a}
\end{minipage}
\begin{minipage}{0.28\textwidth}
\includegraphics[width=\textwidth]{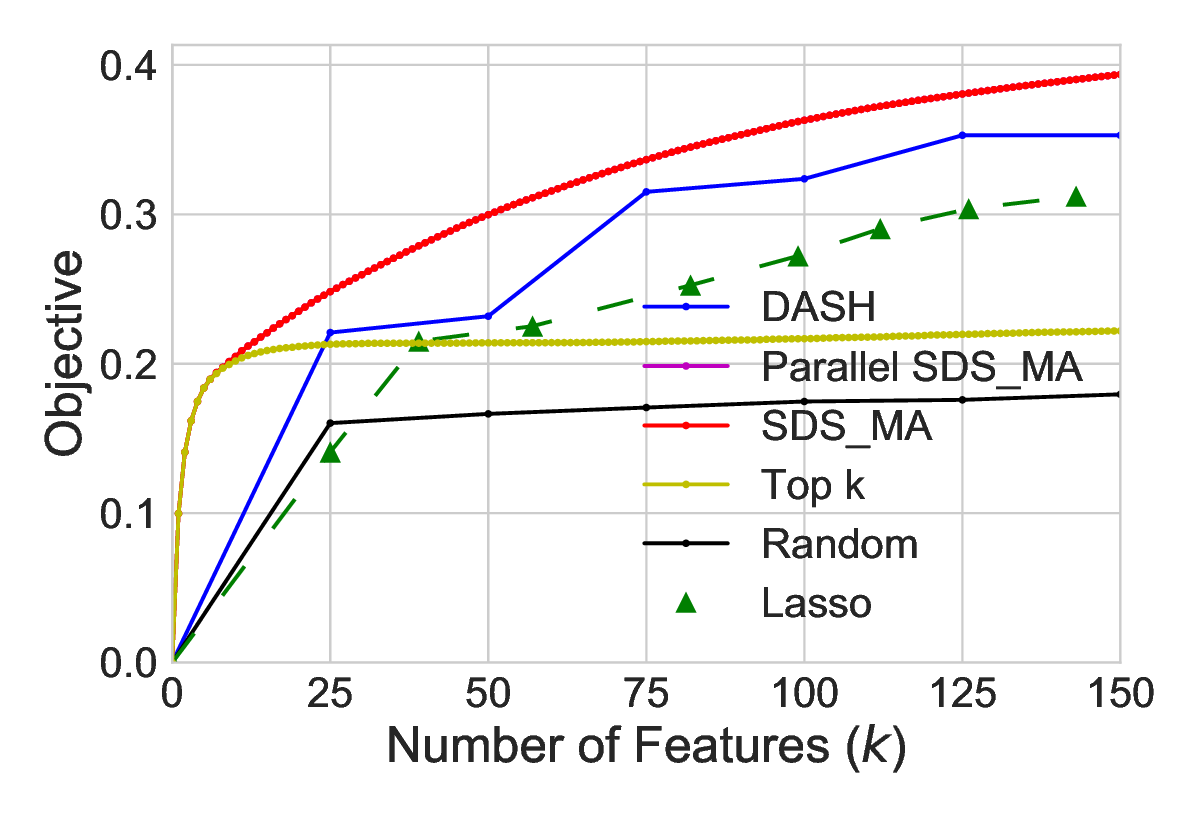}
\subcaption{} \label{fig:lin_b}
\end{minipage}
\begin{minipage}{0.28\textwidth}
\includegraphics[width=\textwidth]{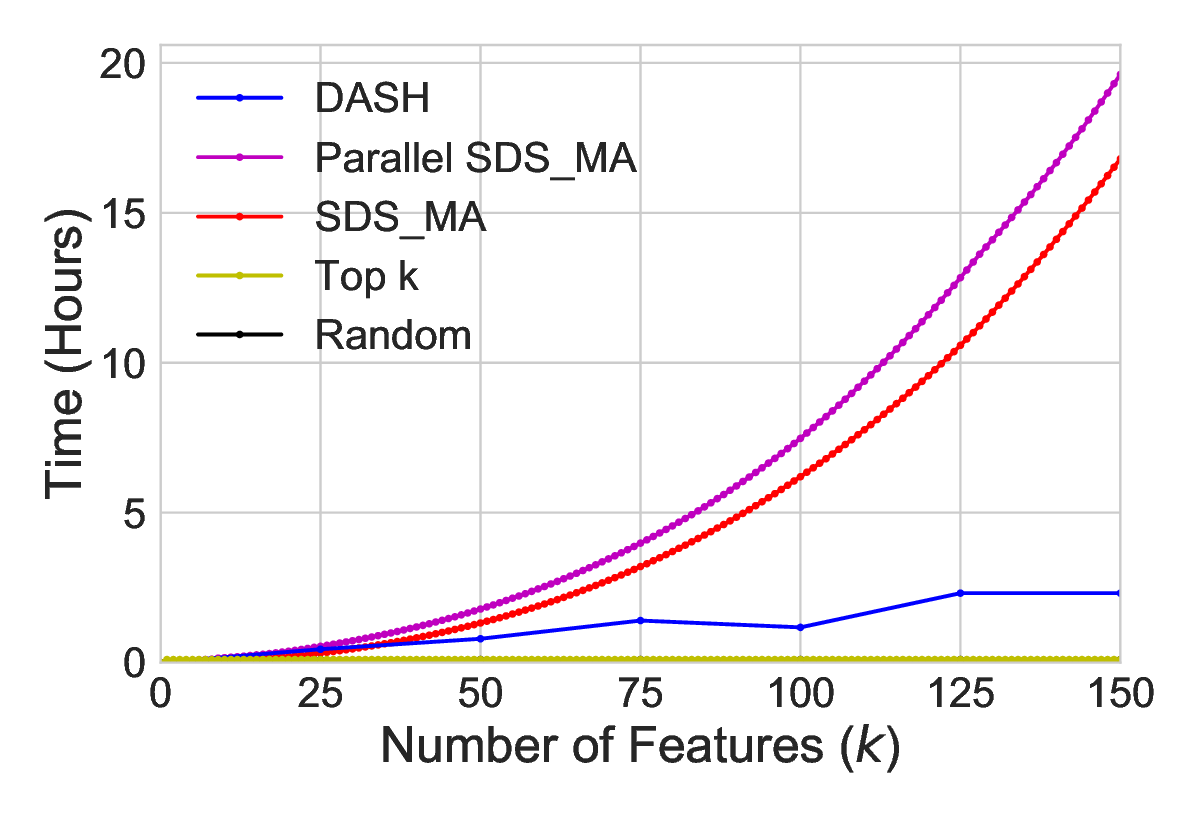}
\subcaption{} \label{fig:lin_c}
\end{minipage}
\begin{minipage}{0.28\textwidth}
\includegraphics[width=\textwidth]{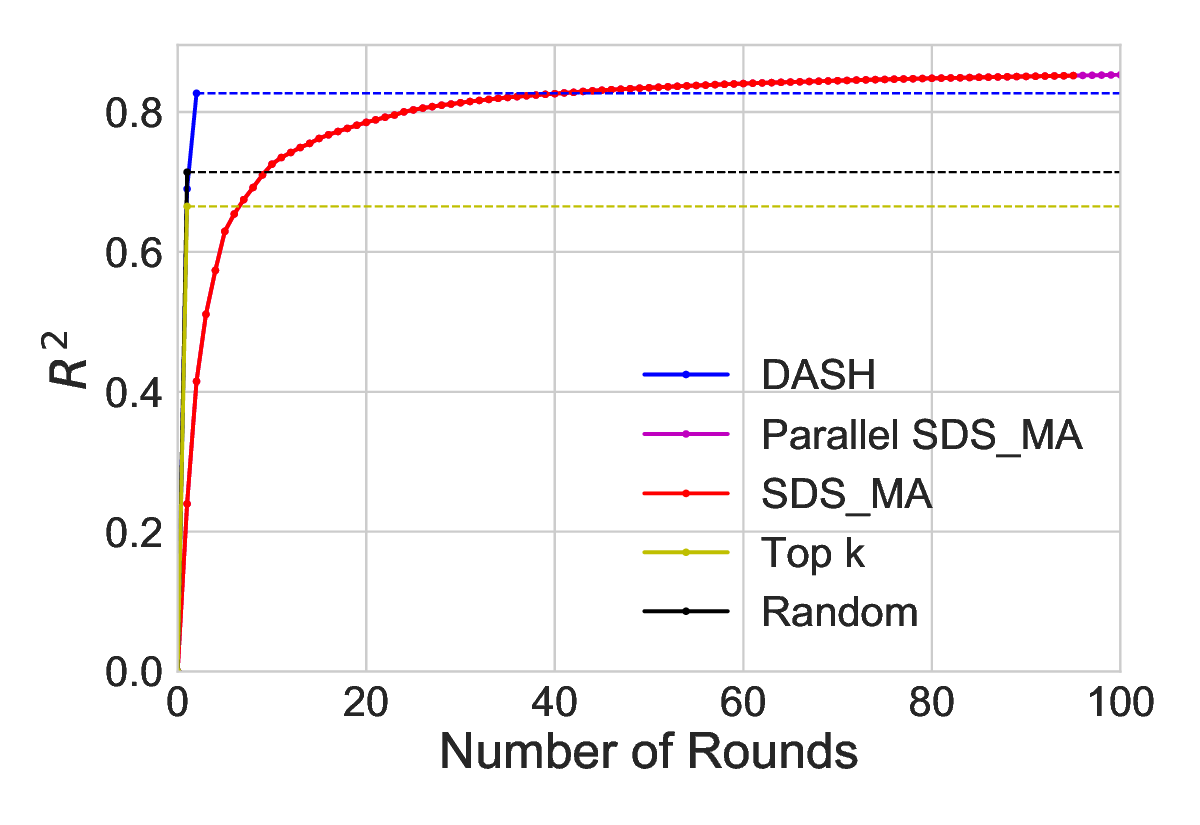}
\subcaption{} \label{fig:lin_d}
\end{minipage}
\begin{minipage}{0.28\textwidth}
\includegraphics[width=\textwidth]{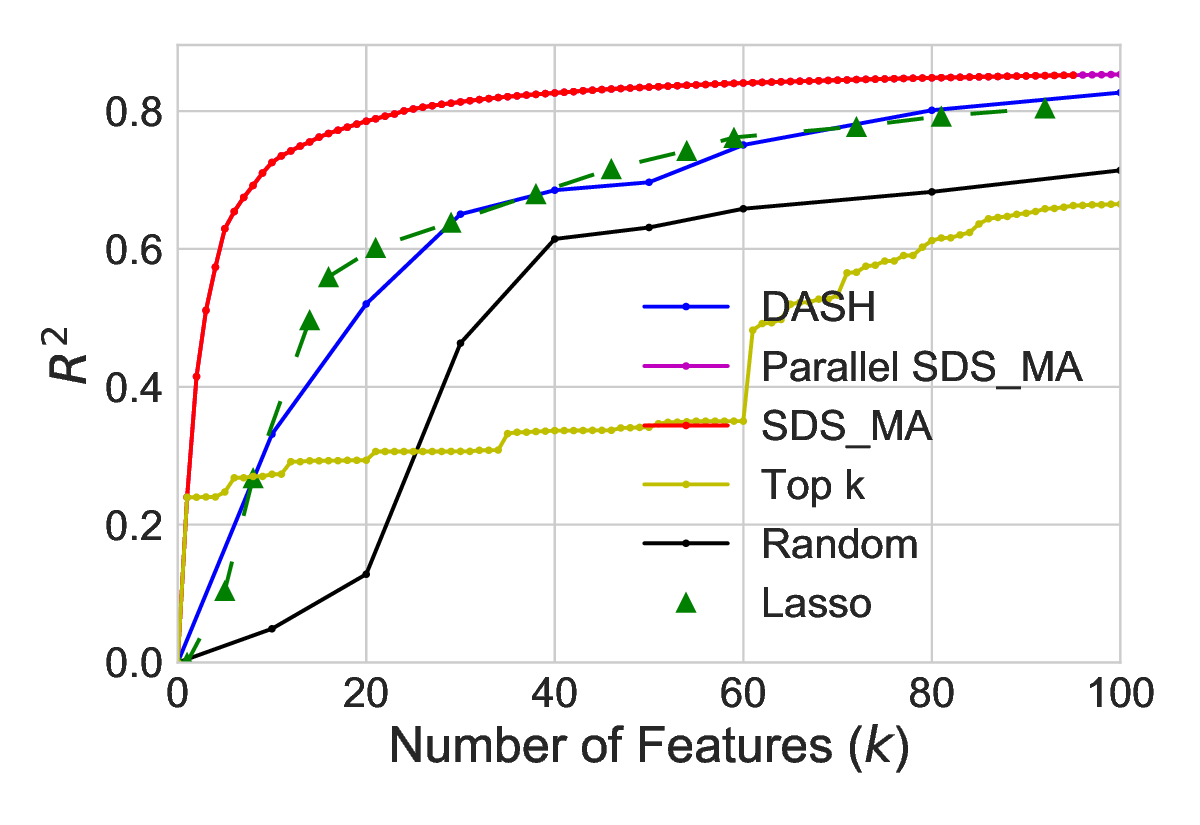}
\subcaption{} \label{fig:lin_e}
\end{minipage}
\begin{minipage}{0.28\textwidth}
\includegraphics[width=\textwidth]{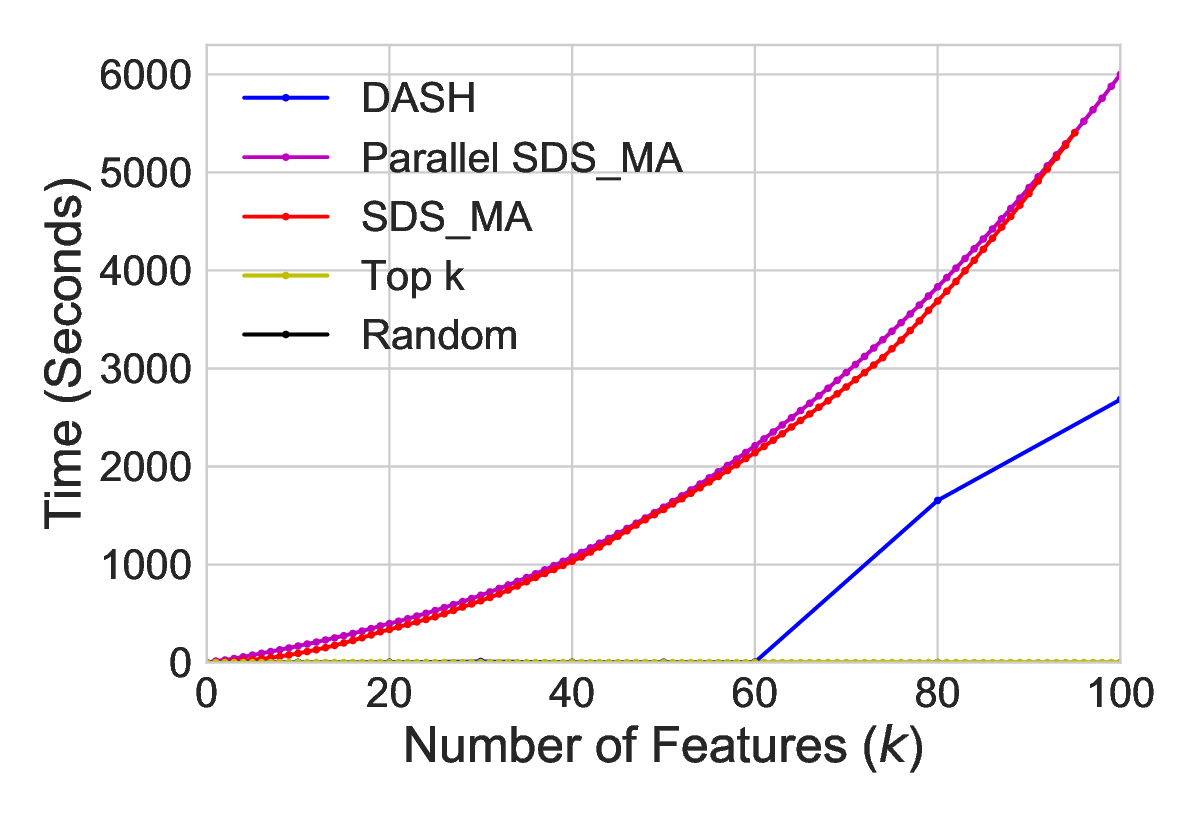}
\subcaption{} \label{fig:lin_f}
\end{minipage}
\caption{Linear regression feature selection results comparing \apx (blue) to baselines on synthetic (top row) and clinical datasets (bottom row). Dashed line represents \textsc{Lasso} extrapolated across $\lambda$.}
\label{fig:lin}
\end{center}
\vskip -0.2in
\end{figure*}

\begin{figure*}[ht]
\vskip 0.2in
\begin{center}
\begin{minipage}{0.28\textwidth}
\includegraphics[width=\textwidth]{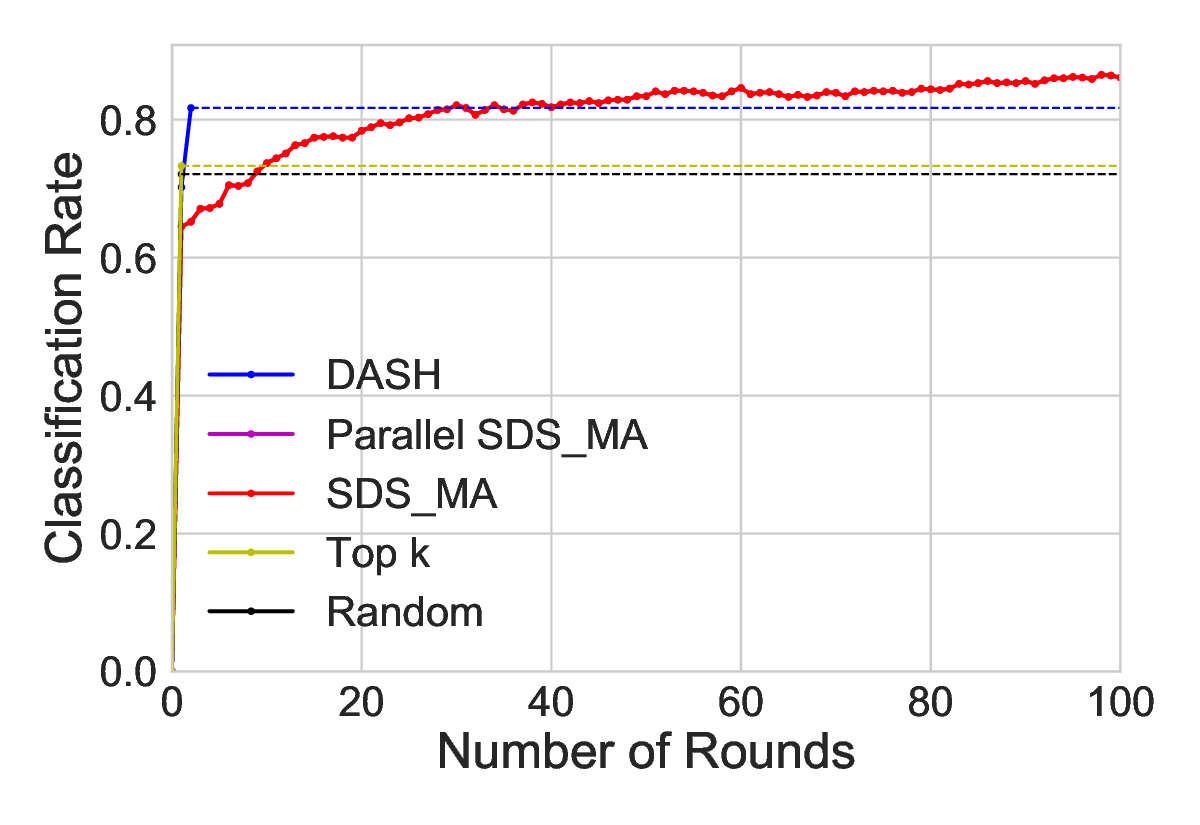}
\subcaption{} \label{fig:log_a}
\end{minipage}
\begin{minipage}{0.28\textwidth}
\includegraphics[width=\textwidth]{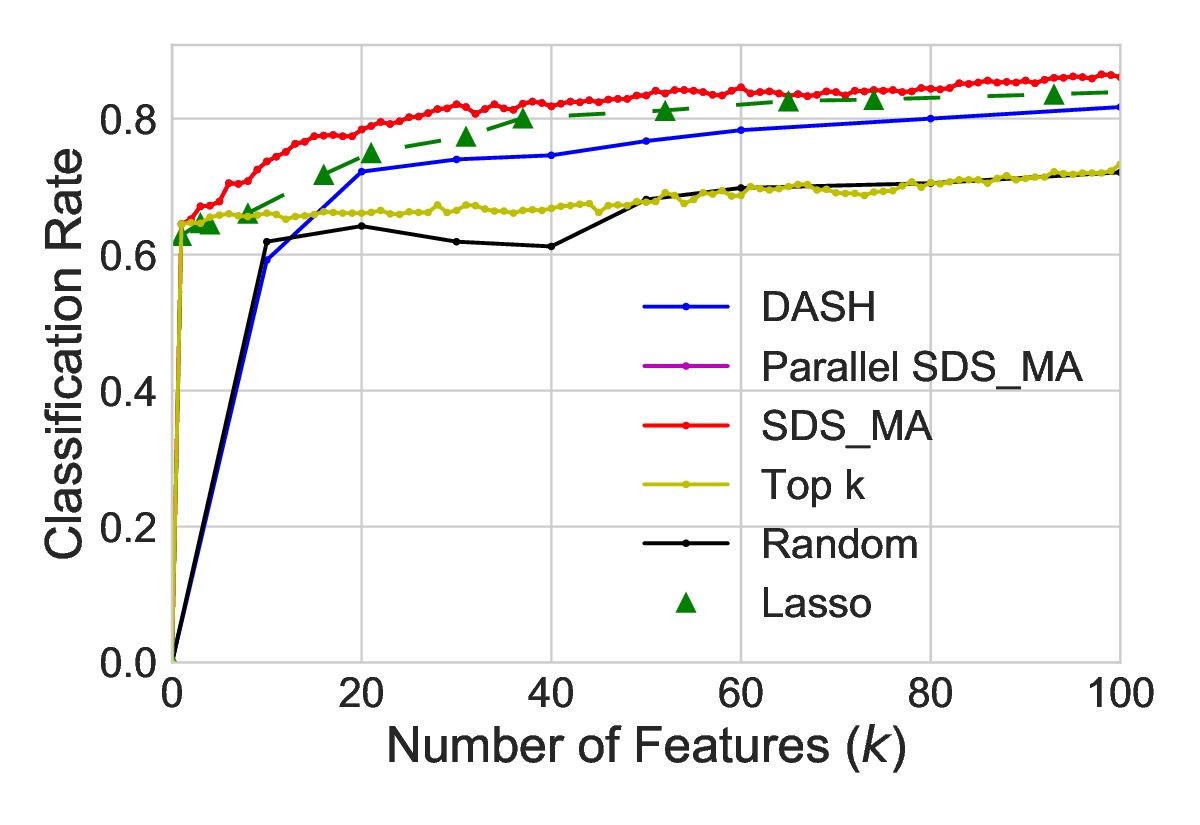}
\subcaption{} \label{fig:log_b}
\end{minipage}
\begin{minipage}{0.28\textwidth}
\includegraphics[width=\textwidth]{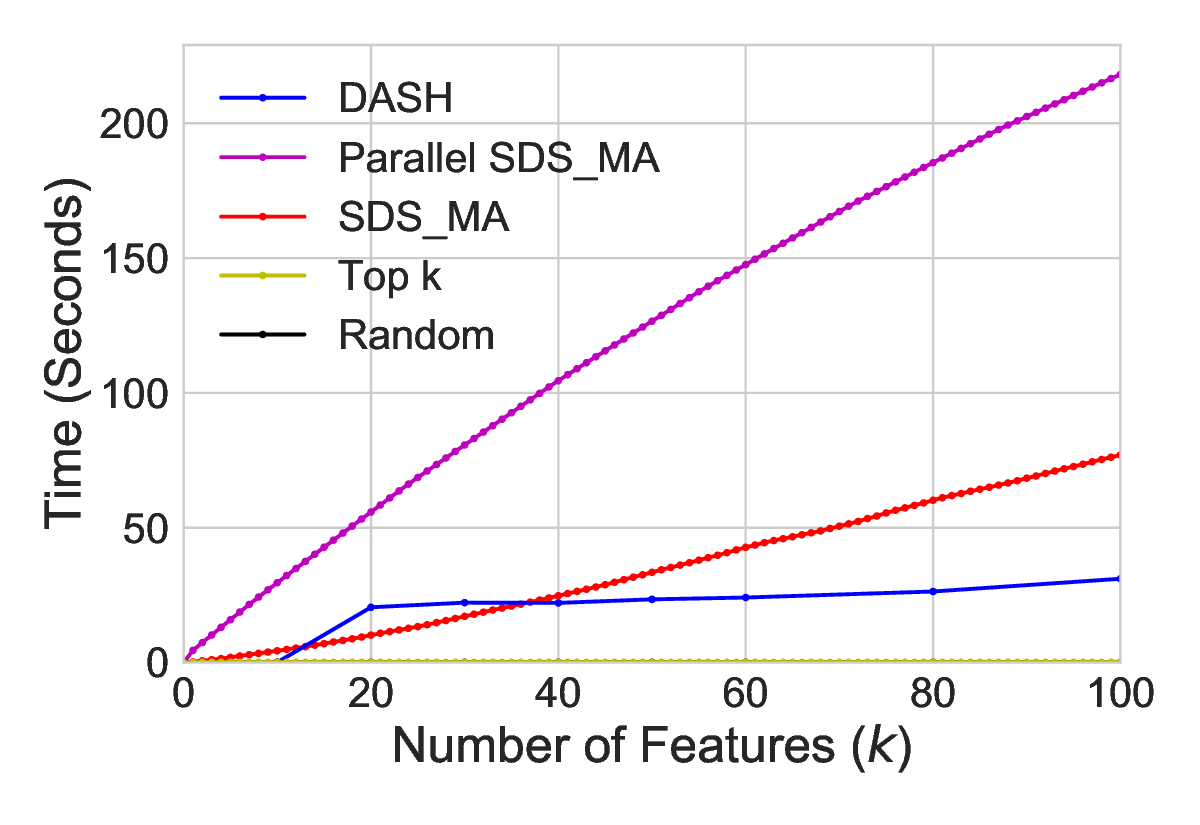}
\subcaption{} \label{fig:log_c}
\end{minipage}

\begin{minipage}{0.28\textwidth}
\includegraphics[width=\textwidth]{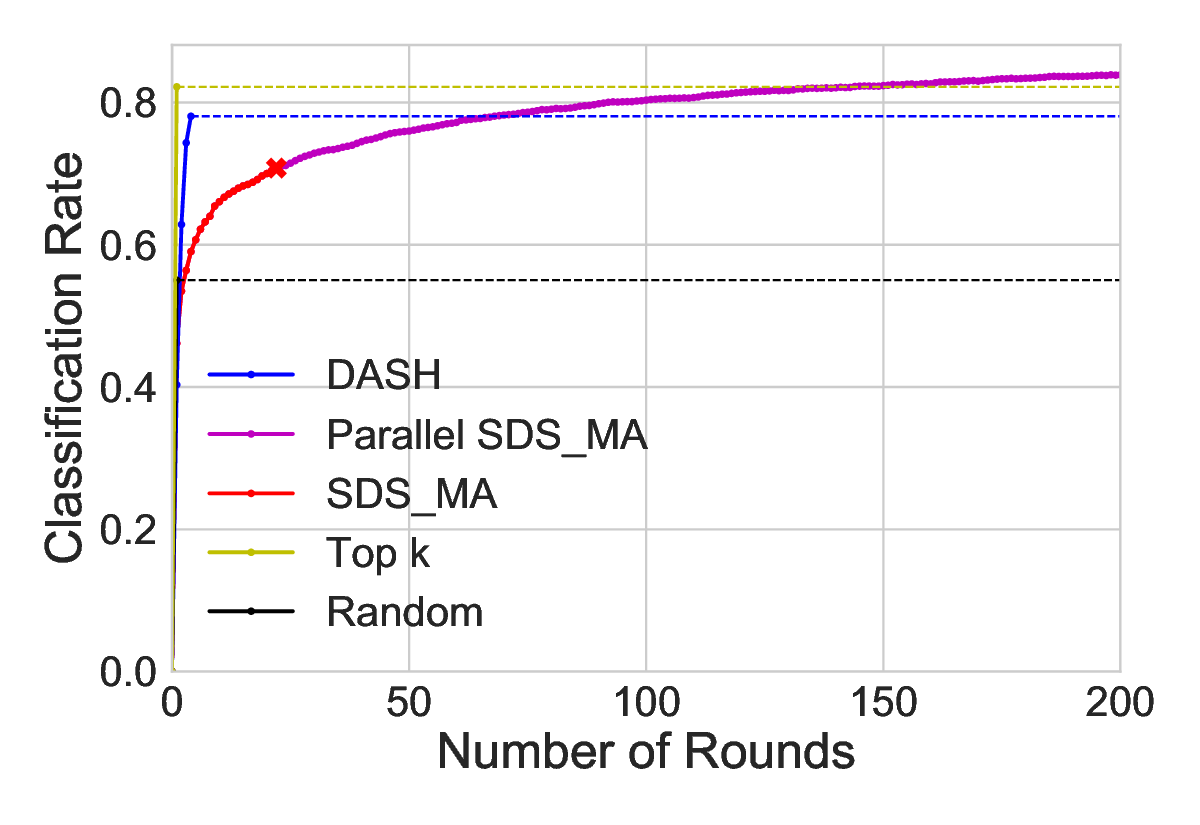}
\subcaption{} \label{fig:log_d}
\end{minipage}
\begin{minipage}{0.28\textwidth}
\includegraphics[width=\textwidth]{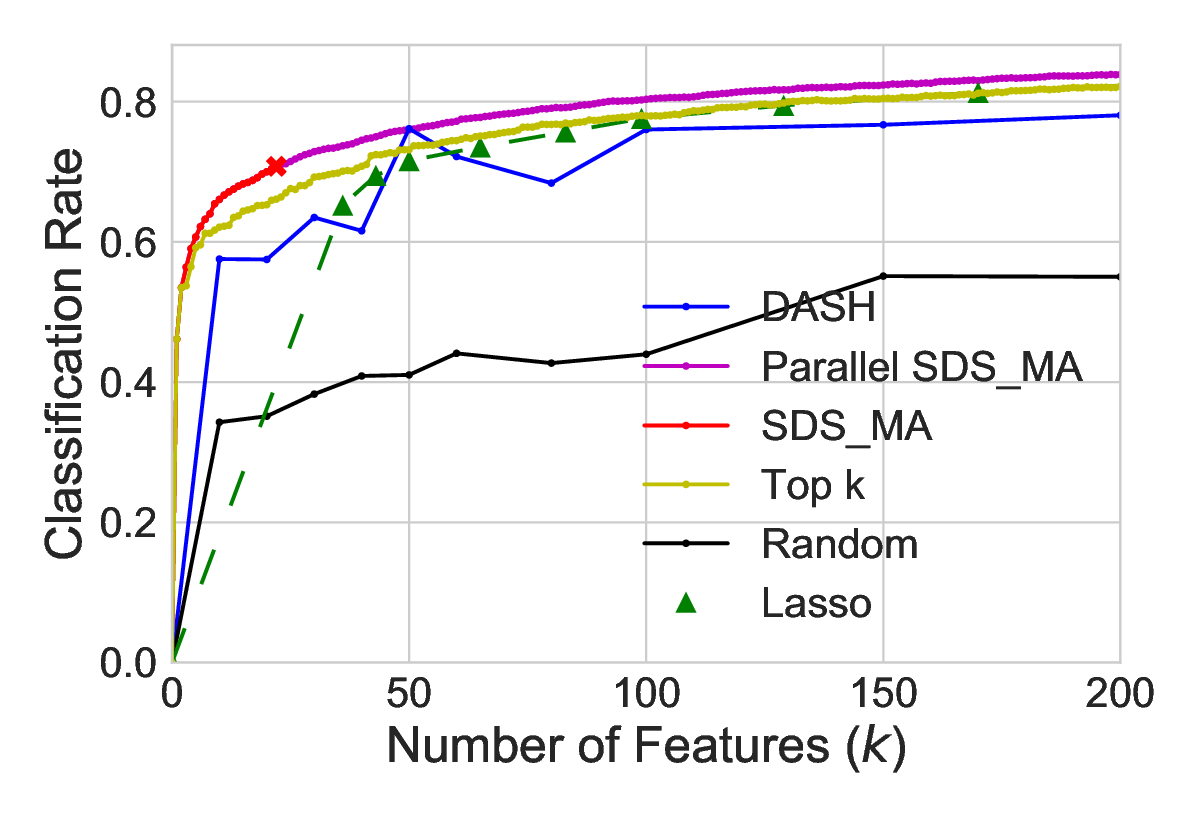}
\subcaption{} \label{fig:log_e}
\end{minipage}
\begin{minipage}{0.28\textwidth}
\includegraphics[width=\textwidth]{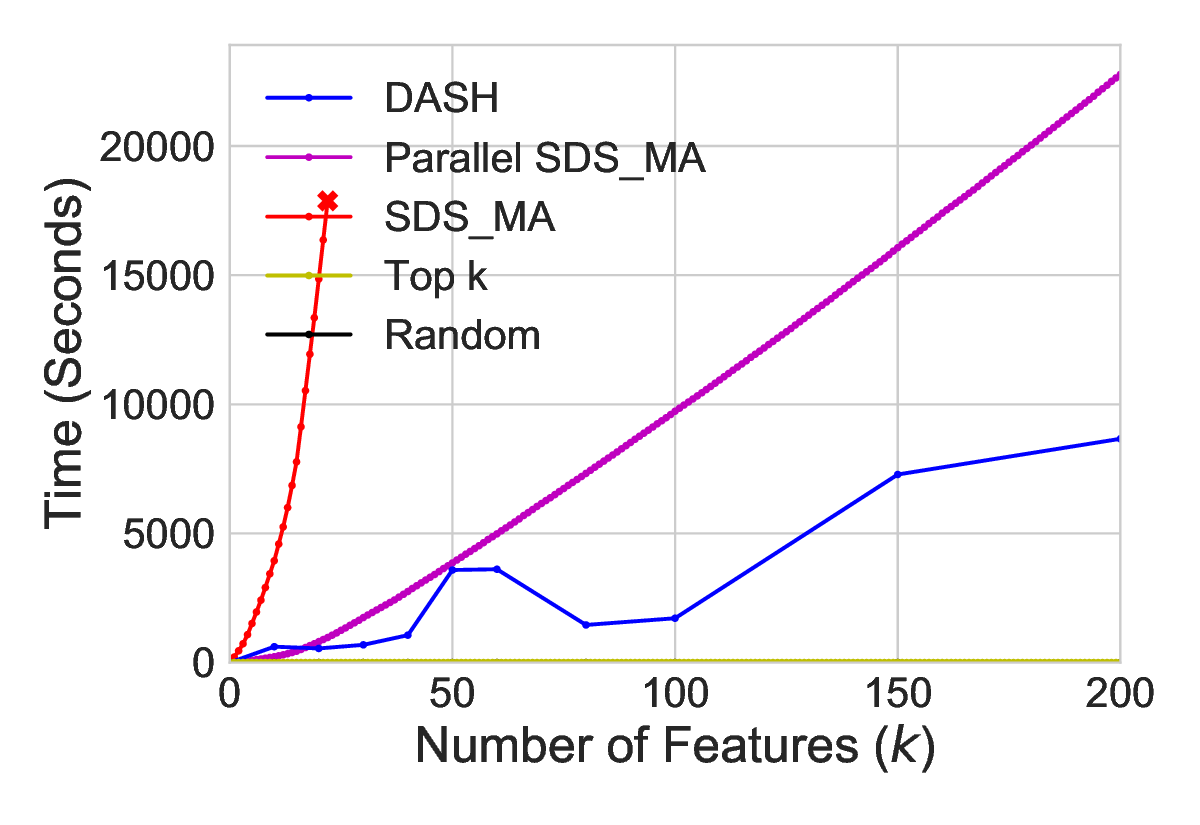}
\subcaption{} \label{fig:log_f}
\end{minipage}
\label{fig:log}
\caption{Logistic regression feature selection results comparing \apx (blue) to baselines on synthetic (top row) and gene datasets (bottom row). The X denotes manual termination of the algorithm due to running time constraints. Dashed line represents approximation for \textsc{Lasso} extrapolated across $\lambda$.}
\end{center}
\vskip -0.2in
\end{figure*}

\section{Experiments} \label{sec:exp}
To empirically evaluate the performance of \textsc{Dash}, we conducted several experiments on feature selection and Bayesian experimental design. While the $1-1/e^{\gamma^4}$ approximation guarantee of \apx is weaker than the $1-1/e^{\gamma}$ of the greedy algorithm ($\textsc{SDS}_{\textsc{MA}}$), we observe that \apx performs comparably to $\textsc{SDS}_{\textsc{MA}}$ and outperforms other benchmarks.  Most importantly, in all experiments, \apx achieves a two to eight-fold speedup of {\bf parallelized} greedy implementations, even for moderate values of $k$. This shows the incredible potential of other parallelizable algorithms, such as adaptive sampling and adaptive sequencing, under the differential submodularity framework.

\paragraph{Datasets.}We conducted experiments for linear and logistic regression using the $\ell_{\texttt{reg}}$ and $\ell_{\texttt{class}}$ objectives, and Bayesian experimental design using $f_{\texttt{A-opt}}$. We generated the synthetic feature space from a multivariate normal distribution. To generate the response variable $\by$, we sample coefficients uniformly (D1) and map to probabilities for classification (D3) and attempt to select important features and samples. We also select features on a clinical dataset $n=385$ (D2) and classify location of cancer in a biological dataset $n=2500$ (D4). We use D1, D2 for linear regression and Bayesian experimental design, and D3, D4 for logistic regression experiments. (See Appendix \ref{app:data} for details.)

\paragraph{Benchmarks.}  We compared \apx to \textsc{Random} (selecting $k$ elements randomly in one round), \textsc{Top-$k$} (selecting $k$ elements of largest marginal contribution), $\textsc{SDS}_{\textsc{MA}}$ \cite{krause2010} and Parallel $\textsc{SDS}_{\textsc{MA}}$, and \textsc{Lasso}, a popular algorithm for regression with an $\ell_1$ regularization term. (See Appendix \ref{app:bench}.)

\paragraph{Experimental Setup.} We run \apx and baselines for different $k$ for two sets of experiments.
\begin{itemize}
\item \textbf{Accuracy vs. rounds.}  In this set of experiments, for each dataset we fixed one value of $k$ ($k=150$ for D1, $k=100$ for D2, D3 and $k=200$ for D4) and ran algorithms to compare accuracy of the solution ($R^2$ for linear regression, classification rate for logistic regression and Bayesian A-optimality for experimental design) as a function of the number of parallel rounds.  The results are plotted in Figures \ref{fig:lin_a}, \ref{fig:lin_d}, Figures \ref{fig:log_a}, \ref{fig:log_d} and Figures \ref{fig:bayes_a}, \ref{fig:bayes_d}; 
\item \textbf{Accuracy and time vs. features.}  In these experiments, we ran the same benchmarks for varying values of $k$ (in D1 the maximum is $k=150$, D2, D3 the maximum is $k=100$ and in D4 the maximum is $k=200$) and measure both accuracy (Figures \ref{fig:lin_b}, \ref{fig:lin_e}, \ref{fig:log_b}, \ref{fig:log_e},  \ref{fig:bayes_b}, \ref{fig:bayes_e}) and time (Figures \ref{fig:lin_c}, \ref{fig:lin_f}, \ref{fig:log_c}, \ref{fig:log_f},  \ref{fig:bayes_c}, \ref{fig:bayes_f}).  When measuring accuracy, we also ran \textsc{Lasso} by manually varying the regularization parameter $\lambda$ to select approximately $k$ features. Since each $k$ represents a different run of the algorithm, the output (accuracy or time) is not necessarily monotonic with respect to $k$.
\end{itemize} 
We implemented \apx with 5 samples at every round.  Even with this small number of samples, the terminal value outperforms greedy throughout all experiments.  The advantage of using fewer samples is that it allows parallelizing over fewer cores.  In general, given more cores one can reduce the variance in estimating marginal contributions which improves the performance of the algorithm.

\begin{figure*}[ht]
\vskip 0.2in
\begin{center}
\begin{minipage}{0.28\textwidth}
\includegraphics[width=\textwidth]{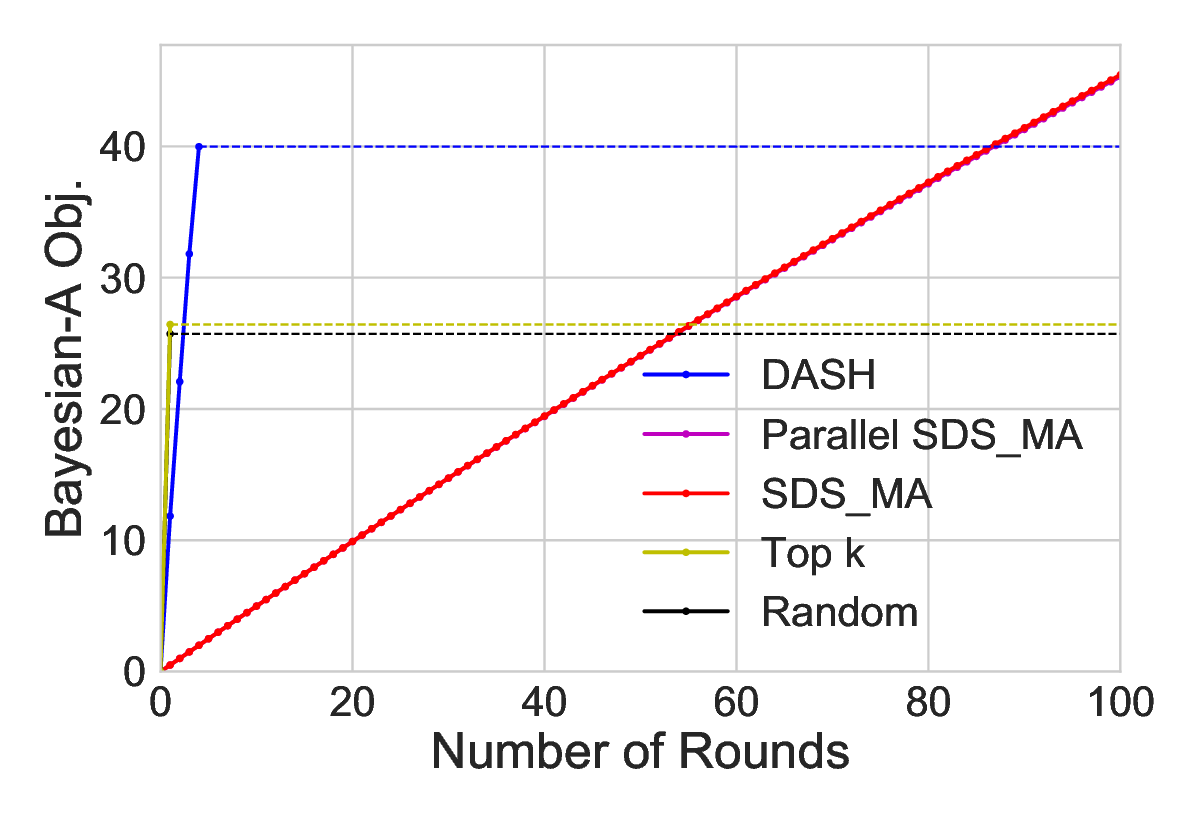}
\subcaption{} \label{fig:bayes_a}
\end{minipage}
\begin{minipage}{0.28\textwidth}
\includegraphics[width=\textwidth]{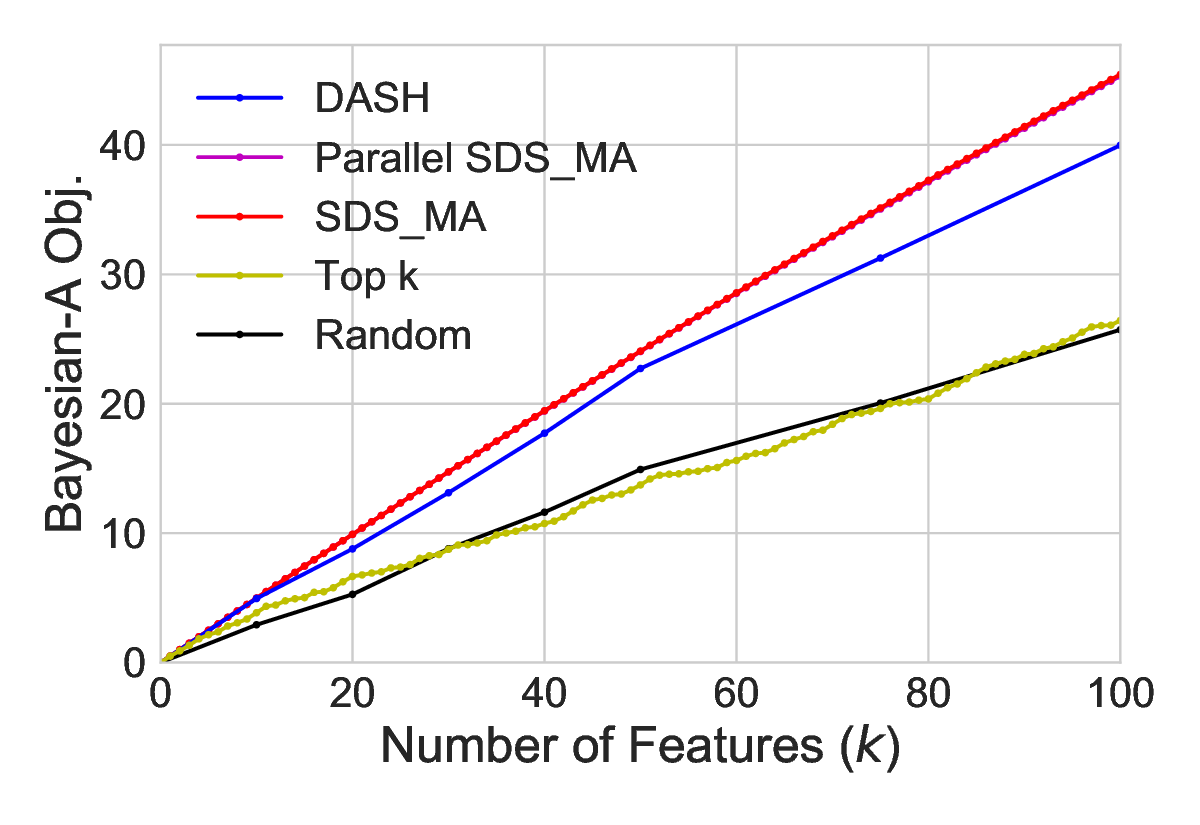}
\subcaption{} \label{fig:bayes_b}
\end{minipage}
\begin{minipage}{0.28\textwidth}
\includegraphics[width=\textwidth]{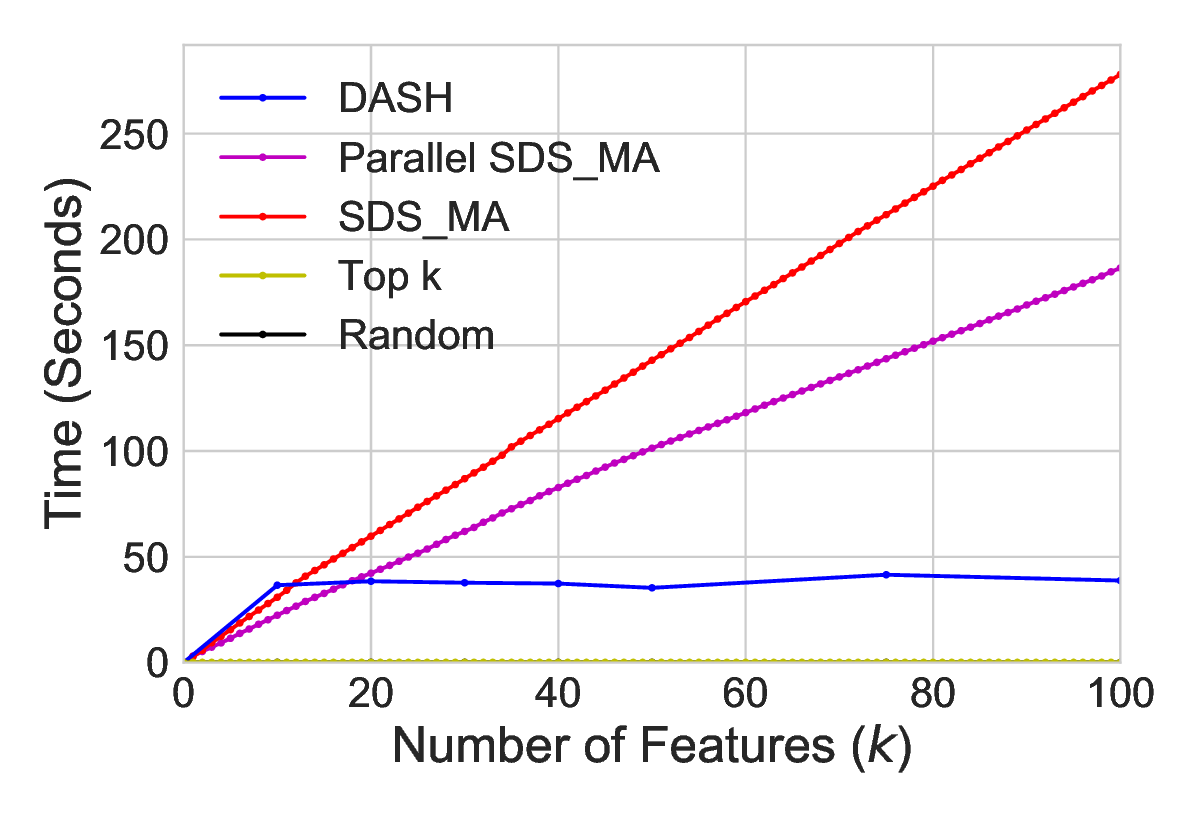}
\subcaption{} \label{fig:bayes_c}
\end{minipage}

\begin{minipage}{0.28\textwidth}
\includegraphics[width=\textwidth]{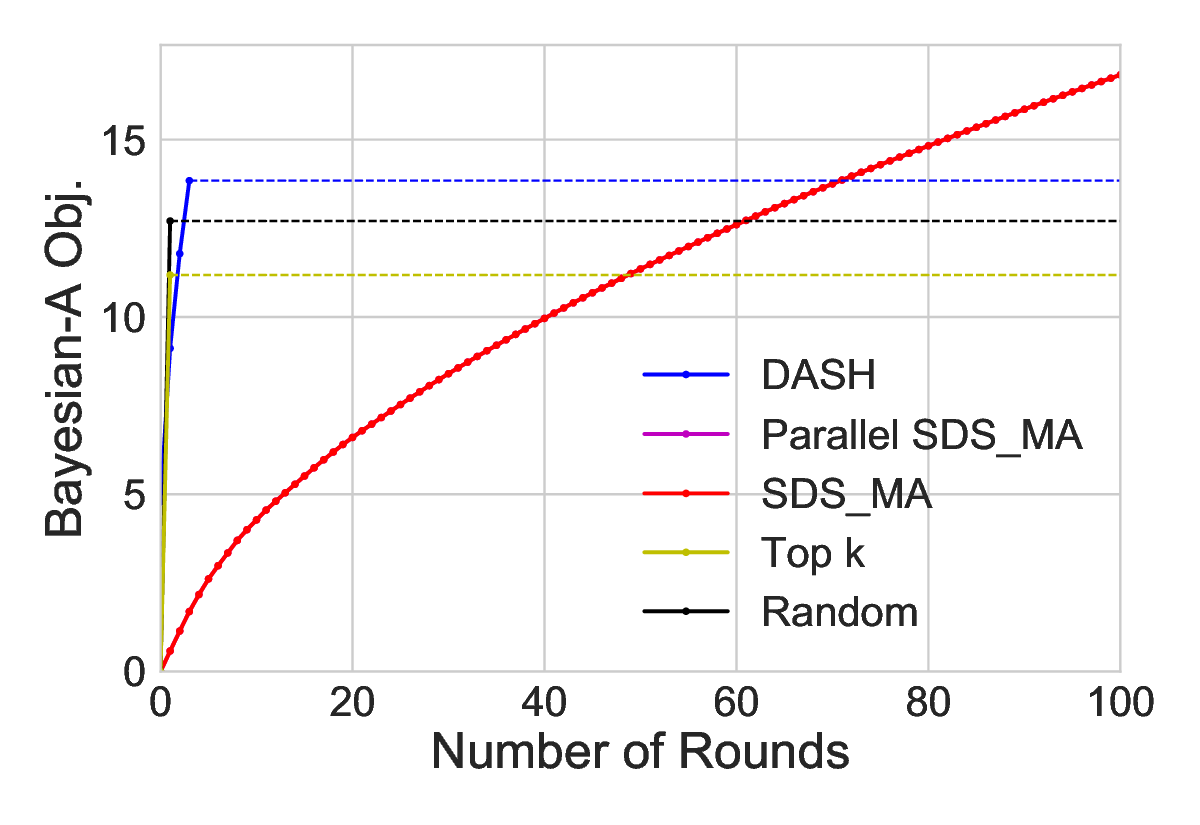}
\subcaption{} \label{fig:bayes_d}
\end{minipage}
\begin{minipage}{0.28\textwidth}
\includegraphics[width=\textwidth]{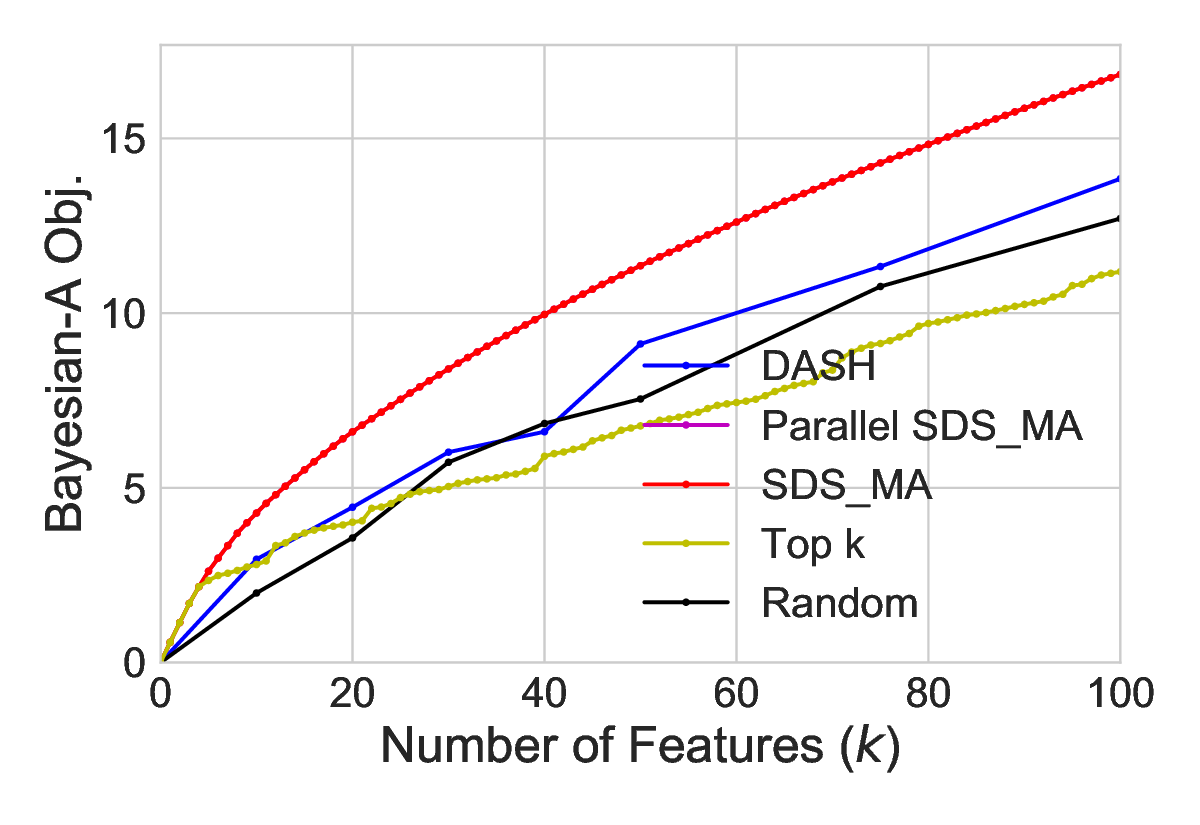}
\subcaption{} \label{fig:bayes_e}
\end{minipage}
\begin{minipage}{0.28\textwidth}
\includegraphics[width=\textwidth]{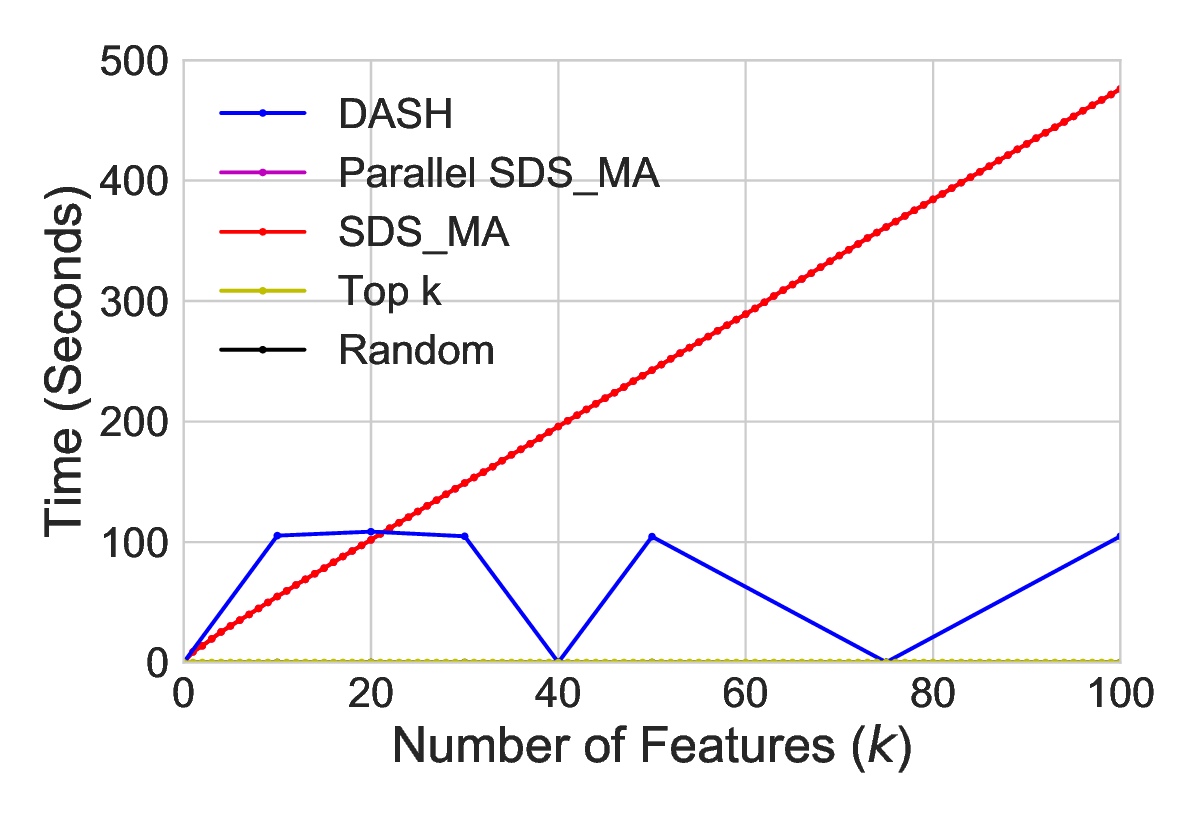}
\subcaption{} \label{fig:bayes_f}
\end{minipage}
\label{fig:bayes}
\caption{Bayesian experimental design results comparing \apx (blue) to baselines on synthetic (top row) and clinical datasets (bottom row).  }
\end{center}
\vskip -0.2in
\end{figure*} 
\paragraph{Results on general performance.} \ We first analyze the performance of \apx. For all applications, Figures \ref{fig:lin_a}, \ref{fig:lin_d}, \ref{fig:log_a}, \ref{fig:log_d}, \ref{fig:bayes_a} and \ref{fig:bayes_d} show that the final objective value of \apx is comparable to $\textsc{SDS}_{\textsc{MA}}$, outperforms \textsc{Top-$k$} and $\textsc{Random}$, and is able to achieve the solution in much fewer rounds. In Figures \ref{fig:lin_b}, \ref{fig:lin_e}, \ref{fig:log_b}, \ref{fig:log_e}, \ref{fig:bayes_b} and \ref{fig:bayes_e}, we show \apx can be very practical in finding a comparable solution set to $\textsc{SDS}_{\textsc{MA}}$ especially for larger values of $k$. In the synthetic linear regression experiment, \apx significantly outperforms \textsc{Lasso} and has comparable performance in other experiments. While \apx outperforms the simple baseline of \textsc{Random}, we note that the performance of \textsc{Random} varies widely depending on properties of the dataset. In cases where a small number of features can give high accuracy, \textsc{Random} can perform well by randomly selecting well-performing features when $k$ is large (Figure \ref{fig:lin_e}). However, in more interesting cases where the value does not immediately saturate, both \apx and $\textsc{SDS}_{\textsc{MA}}$ significantly outperform $\textsc{Random}$ (Figure \ref{fig:lin_b}, \ref{fig:bayes_b}).

We can also see in Figures \ref{fig:lin_c}, \ref{fig:lin_f}, \ref{fig:log_c}, \ref{fig:log_f},  \ref{fig:bayes_c} and \ref{fig:bayes_f} that \apx is computationally efficient compared to the other baselines. In some cases, for smaller values of $k$, $\textsc{SDS}_{\textsc{MA}}$ is faster (Figure \ref{fig:log_c}). This is mainly due to the sampling done by \apx to estimate the marginals, which can be computationally intensive. However, in most experiments, \apx terminates more quickly even for small values of $k$. For larger values, \apx shows a two to eight-fold speedup compared to the fastest baseline.

\paragraph{Effect of oracle queries.} \ Across our experiments, the cost for oracle queries vary widely. When the calculation of the marginal contribution is computationally cheap, parallelization of $\textsc{SDS}_{\textsc{MA}}$ has a longer running time than its sequential analog due to the cost of merging parallelized results (Figures \ref{fig:lin_c}, \ref{fig:log_c}). However, in the logistic regression gene selection experiment, calculating the marginal contribution of an element to the solution set can span more than 1 minute.  In this setting, using sequential $\textsc{SDS}_{\textsc{MA}}$ to select 100 elements would take several days for the algorithm to terminate (Figure \ref{fig:log_f}). Parallelization of $\textsc{SDS}_{\textsc{MA}}$ drastically improves the algorithm running time, but \apx is still much faster and can find a comparable solution set in under half the time of parallelized $\textsc{SDS}_{\textsc{MA}}$.

In both cases of cheap and computationally intensive oracle queries, \apx terminates more quickly than the sequential and parallelized version of $\textsc{SDS}_{\textsc{MA}}$ for larger values of $k$. This can be seen in Figures \ref{fig:lin_c}, \ref{fig:log_c} and \ref{fig:bayes_c} where calculation of marginal contribution on synthetic data is fast and in Figures \ref{fig:lin_f}, \ref{fig:log_f} and \ref{fig:bayes_f} where oracle queries on larger datasets are much slower. This shows the incredible potential of using \apx across a wide array of different applications to drastically cut down on computation time in selecting a large number elements across different objective functions. Given access to more processors, we expect even a larger increase in speedup for \textsc{Dash}. 

\newpage
\section*{Acknowledgements}
The authors would like to thank Eric Balkanski for helpful discussions. This research was supported by a Smith Family Graduate Science and Engineering Fellowship, NSF grant CAREER CCF 1452961, NSF CCF 1301976, BSF grant 2014389, NSF USICCS proposal 1540428, a Google Research award, and a Facebook research award.
\bibliography{diff_sm}
\bibliographystyle{alpha}

\newpage
\appendix
\section{Motivational Examples}
\subsection{\textsc{Adaptive-Sampling} does not work for weakly submodular functions}\label{appendix:toy1}
To demonstrate why adding sets at each iteration can perform badly compared to adding single elements, we construct a weakly submodular function where greedy can achieve the optimal value and the performance of adding sets of elements to the solution set can be poor. The construction is a slight variant of the one in \cite{elenberg2017}.

We have a ground set consisting of two types of elements, $\mathcal N = \{U, V\}$, where $U = \{u_i\}_{i=1}^k$ and $V = \{v_i\}_{i=1}^k$. For every subset $S \subseteq \mathcal N$, $u(S) = | S \cap U|$ and $v(S) = | S \cap V|$. Now, we define the following set function 
$$f(S) = \min \{2\cdot u(S) + 1, 2\cdot v(S)\}, \quad \forall S \subseteq N.$$ 

For cardinality constraint $k$, we can see that the optimal solution is $k$.
\begin{lemma}\label{lemma:f}
$f$ is nonnegative, monotone and 0.5-weakly submodular \cite{elenberg2017}.
\end{lemma}

For simplicity, assume the number of rounds $r=1$. We now show why \textsc{Adaptive Sampling} performs poorly. In the first step, \textsc{Adaptive Sampling} will filter out elements with low marginal contributions. Since $f(u_i) = 0$ and $f(v_i)=1$ for all $i$, by standard concentration bounds, elements of $U$ will be filtered out and only elements in $V$ will remain. Now, the algorithm attempts to add a set of $k$ elements into the solution set. Since all subsets of $V$ have a value of 1, the algorithm can only achieve a value of 1 even when the optimal value is $k$. As $k$ increases, this algorithm performs arbitrarily poorly.

\subsection{Existing adaptive algorithms fail for differentially submodular functions}\label{appendix:toy2}
\textsc{Adaptive-Sampling} \cite{balkanski2019} for submodular functions does not guarantee termination for differentially submodular functions. The filtering step that removes elements with low individual marginal contribution does not guarantee the marginal contribution of the set of ``good'' elements is larger than the threshold value as in the submodular case. For differentially submodular functions, the algorithm may result in an infinite \texttt{while} loop, where ``bad'' elements are filtered out, but no combination of remaining elements adds sufficient value to the solution set. We show two examples.

We use the construction defined in the previous section $f(S) = \min \{2\cdot u(S) + 1, 2\cdot v(S)\}$. While $f(S)$ is weakly submodular, we note that it is not differentially submodular. Consider the case where $S=\{u_1 \}$ and $A = \{ v_i\}_{i=1}^n$, then $\sum_{a\in A} f_S(a) = n$, but $f_S(A) = 1$. However, we can show a modified function is differentially submodular on small set sizes, which is sufficient for our example. Let $f'(S) = f(S)$ where $|S| \leq 2$, then $f'$ is 0.25-differentially submodular. This construction demonstrates a simple case of how adaptive sampling on submodular functions fails for differentially submodular functions.

\begin{lemma}
$f'$ is 0.25-differentially submodular.
\end{lemma}
\begin{proof}
We can use the lower bound from weak submodularity of $f$ from Lemma \ref{lemma:f}, which holds for $f'$: 
$$\sum_{a\in A} f'_S(a) \geq 0.5 \cdot f'_S(A)$$

With our modification, we can now lower bound by marginal contributions:

$$f'_S(A) \geq \frac{1}{|A \backslash S|} \cdot \sum_{a\in A} f'_S(a) \geq 0.5 \cdot \sum_{a\in A} f'_S(a)$$
which shows that $f'$ is 0.25-differentially submodular.
\end{proof}

We now show that \textsc{Adaptive-Sampling} does not guarantee termination for differentially submodular functions. 

For simplicity, let $\epsilon = 0$. We wish to select $k$ elements to achieve the optimal solution of $k$ by adding 2 elements at a time to the solution set using \textsc{Adaptive-Sampling} on $f'(S)$. We note that \apx reduces to \textsc{Adaptive-Sampling} when $\alpha=1$.

To survive the filtering step, each element must have a marginal contribution of 1. Since $f'(v_i) = 1$ and $f'(u_i) = 0$ for all $i$, only elements in $V$ are labeled as ``good'' by the algorithm. The elements in $U$ are filtered out. Then the algorithm attempts to add 2 elements from $V$ into the solution set and expects that the marginal contribution of the set has value 2 for termination (for $\alpha=1$). However, this is not the case, as $f'(v_i \cup v_j) = 1$ and \textsc{Adaptive-Sampling} enters an infinite \texttt{while} loop by failing to find a set with large enough marginal contribution. 

However, \apx will terminate. By adding a factor of $\alpha^2$ to lower the threshold, \apx accepts the set of 2 elements in $V$ and successfully adds these 2 elements into the solution set. The algorithm leverages the fact that differential submodularity both lower bounds the elements that are added into the set and upper bounds the values of elements that are filtered out. 

In another more concrete example, we show that after individual elements are filtered out, there is no set of elements that will pass the \textsc{Adaptive-Sampling} threshold to be added into the solution set. This results in an infinite \texttt{while} loop.

Consider the following variables in the context of the $R^2$, goodness-of-fit objective (See Appendix \ref{appendix:eig} for more details):

\begin{eqnarray}
 \by &=& \begin{bmatrix} 1 & 0 & 0  & 0 \end{bmatrix}^\top \nonumber\\
 \bx_1 &=& \begin{bmatrix} 0 & 1 & 0  & 0 \end{bmatrix}^\top \nonumber\\
 \bx_2 &=& \begin{bmatrix} 0 & 0 & 1  & 0 \end{bmatrix}^\top \nonumber\\
 \bx_3 &=& \begin{bmatrix} 0 & 0 & 0  & 1 \end{bmatrix}^\top \nonumber\\
  \bx_4 &=& \begin{bmatrix} \sqrt{\frac{1}{2}} & \sqrt{\frac{1}{2}} & 0  & 0 \end{bmatrix}^\top \nonumber\\
 \bx_5 &=& \begin{bmatrix} \sqrt{\frac{1}{2}} & 0 & \sqrt{\frac{1}{2}}  & 0 \end{bmatrix}^\top \nonumber\\
 \bx_6 &=& \begin{bmatrix} \sqrt{\frac{1}{2}} & 0 & 0  & \sqrt{\frac{1}{2}} \end{bmatrix}^\top \nonumber
 \end{eqnarray}

We wish to choose two features $\bx_i$ that best estimate $\by$ (and maximize $R^2$). We can see that the optimal solution of $R^2=1$ is achieved by three different 2-subsets: $(\bx_1, \bx_4), (\bx_2, \bx_5), (\bx_3, \bx_6)$. For $\bx_1, \bx_2, \bx_3$, the marginal contribution is $R^2 = 0$. For $\bx_4, \bx_5, \bx_6$, $R^2 = \frac{1}{2}$. 

For simplicity, let $\epsilon=0$, $r=1$ and $f(O)=1$. \textsc{Adaptive-Sampling} will first filter out $\bx_1, \bx_2,\bx_3$ because the marginal contribution is less than $\frac{1}{2}$. Then it will attempt to select 2 elements from $\bx_4, \bx_5, \bx_6$ to comprise the solution set. The while loop will only terminate once it finds a 2-subset where the marginal contribution is larger or equal to 1. However, due to the non-submodular properties of the objective, even though the bad elements were filtered out, the marginal contribution of any 2-subset from $\bx_4, \bx_5, \bx_6$ does not achieve the necessary threshold value. The $R^2$ of any 2-subset from $\bx_4, \bx_5, \bx_6$ is $\frac{2}{3}$. As an example, let us calculate the marginal contribution of $\bx_4$ and $\bx_5$.
\begin{eqnarray}
R^2_{4,5}&=& (\by^\top \bX_{4,5}) (\bX^\top_{4,5} \bX_{4,5} )^{-1}(\bX^\top_{4,5} \by) \nonumber \\
&=& \frac{4}{3} \begin{bmatrix} \sqrt{\frac{1}{2}} & \sqrt{\frac{1}{2}} \end{bmatrix} 
\begin{bmatrix} 
1 & -\frac{1}{2} \\
-\frac{1}{2} & 1 
 \end{bmatrix} 
\begin{bmatrix} 
\sqrt{\frac{1}{2}} \\
\sqrt{\frac{1}{2}}
 \end{bmatrix} = \frac{2}{3} < 1  \nonumber
 \end{eqnarray}
Thus, \textsc{Adaptive-Sampling} will enter an infinite while loop and never terminate.

We note that greedy achieves the optimal solution by first selecting a feature from $\bx_4, \bx_5, \bx_6$ in the first iteration and then selecting the second feature from $\bx_1, \bx_2, \bx_3$. 

\section{Notions of Approximate Submodularity}\label{appendix:notions}
In this section, we discuss related work on notions of non-submodularity and their theoretical guarantees on choosing a set of size $k$ to comprise the solution set. Our definition differs from these notions in three aspects and allows for parallelization. Specifically, 1) we bound the marginal contribution of the objective function $f$ and not just the function value and 2) we consider the marginal contribution of sets of elements instead of a singleton and 3) we allow the flexibility of being bound by two different submodular functions. These alterations are necessary for the proof of our low-adaptivity algorithm.

Krause et al. \cite{krause2010} define {\it approximate submodularity} with parameter $\epsilon \geq 0$ as functions that 
satisfy an additive approximate diminishing returns property, i.e. $\forall S \subseteq T \subseteq N \backslash a$ it holds
that $f_S(a) \geq f_T(a) - \epsilon$. $\textsc{SDS}_{\textsc{MA}}$ applied to functions with this additive property inherits an additive guarantee of $f(S) \geq (1-1/e)f(O) - k\epsilon$.

Das and Kempe {\cite{das2011} define the {\it submodularity ratio} with parameter $\gamma \geq 0$ to quantify how close a function is 
to submodularity, where $ \gamma = \min_{S,A} \frac{\sum_{a\in A} f_S(a)}{f_S(A)}$. Elenberg et al. \cite{elenberg2018} extend their work and lower bound the submodularity ratio using strong concavity and smoothness parameters for generalized linear models. $\textsc{SDS}_{\textsc{MA}}$ applied to functions with this property inherits a guarantee of $f(S) \geq (1-1/e^\gamma)f(O)$. Because $\gamma$ is difficult to compute on a real dataset (only possible using brute force), Bian et al. \cite{bian2017} 
introduce the Greedy submodularity ratio $\gamma^G = \min_{A: |A| =k , S^t} \frac{\sum_{a\in A} f_{S^t}(a)}{f_{S^t}(A)}$,
where $S^t$ is the set chosen by the greedy algorithm at step $t$.

For multiplicative bounds, Horel et al. \cite{horel2016} define {\it $\epsilon$-approximately submodular} functions where $f$ is approximately submodular if there exists a submodular function $g$ s.t. 
$ (1-\epsilon) g(S) \leq f(S) \leq (1+\epsilon)g(S), \forall S \subseteq N$. In this definition, the function is approximated pointwise by a submodular function, not its marginals as in differential submodularity. Gupta et al. \cite{gupta2018} define a similar property on the marginals of the function where $f$ is {\it $\delta$-approximately submodular}
if there exists a submodular function $g$ s.t. 
$ (1-\delta) g_S(a) \leq f_S(a) \leq (1+\delta)g_S(a), \forall S \subseteq N, a \notin S$. Differential submodularity generalizes this definition so that the functions that bound the objective can differ. This is necessary in cases where the objective function contains a diversity factor.

\section{Relationship to PRAM} \label{app:pram}
The PRAM model is a generalization of the RAM model with parallelization. It represents an idealized model that can execute instructions in parallel with any number of processors in a shared memory machine. In this framework, the notion of depth is closely related to the one of adaptivity that we discuss in this paper. The \emph{depth} of a PRAM model is the number of parallel steps in an algorithm or the longest chain of dependencies. The area of designing low-depth algorithms have been extensively studied. Our results extend to the PRAM model, similarly to the results of the original adaptive sampling algorithm for submodular maximization. For more detail, please see Appendix A.2.2 of \cite{balkanski2018}.

\section{Bayesian Experimental Design Details} \label{app:bayes}

In Bayesian experimental design, we would like to select a set of experiments to optimize some statistical criterion. Specifically, the Bayesian A-optimality criterion is used to maximally reduce the variance in the posterior distribution over the parameters. 

More formally, let $n$ experimental stimuli comprise the matrix $\bX \in \rrr^{d\times n}$, where each experimental stimuli $\bx_i\in \rrr^d$ is a column in $\bX$. We can select a set $S\subseteq \mathcal N$ of stimuli and denote this as $\bX_S \in \rrr^{d\times |S|}$. Let $\btheta \in \rrr^d$ be the parameter vector in the linear model $\by_S = \bX^T_S \btheta + \bw$, where $\bw \sim \mathcal N(0,\sigma^2\bI)$ is noise from a Gaussian distribution, $\by_S$ is the vector of dependent variables, and $\btheta \sim \mathcal N(0, \bLambda^{-1} ), \bLambda = \beta^2\bI$ is the prior that takes the form of an isotropic Gaussian. Then,

$$\begin{bmatrix}
\by_S \\
\btheta
\end{bmatrix} \sim  \mathcal N(0,\bSigma), \bSigma =  
\begin{bmatrix}
\sigma^2 \bI + \bX^T_S\bLambda^{-1}\bX_S & \bX^T_S\bLambda^{-1} \\
\bLambda^{-1}\bX_S & \bLambda^{-1} 
\end{bmatrix} 
$$
which implies $\bSigma_{\btheta | \by_S} = (\bLambda + \sigma^{-2} \bX_S \bX^T_S)^{-1}$.

Now, we can define our A-optimality objective as 
\begin{eqnarray}
f_{\texttt{A-opt}}(S) = \Tr(\Sigma_\btheta) - \Tr(\bSigma_{\btheta | \by_s} = \Tr(\bLambda^{-1}) - \Tr((\bLambda + \sigma^{-2}\bX_S \bX_S^T)^{-1})
\end{eqnarray}
To regularize for diverse experiments, we can formulate the problem as follows
$$ \max_{S:|S| \leq k} f_{\texttt{A-div}}(S) = f_{\texttt{A-opt}}(S) + d(S),$$
where $d: 2^N \rightarrow \rrr_+$ is a ``diverse" submodular function promoting regularization.

Krause et al. \cite{krause2008} has shown that the Bayesian A-optimality objective is not submodular and Bian et al. \cite{bian2017} has shown that submodularity ratio of the objective can be lower bounded. With the traditional greedy algorithm, we get a $1-1/e^\gamma$ approximation guarantee, where $\gamma \geq \frac{\beta^2}{\|\bX\|^2(\beta^2+\sigma^{-2} \|\bX\|^2)}$ \cite{bian2017}.
\section{Missing Proofs from Section \ref{sec:obj}}\label{appendix:main2}
\subsection{Proof of Corollary \ref{corollary:lin}}
\begin{proof}
In the case where there is no diversity regularization term, the concavity and smoothness parameters correspond to the sparse eigenvalues of the covariance matrix, i.e., $m_{k} = \lambda_{min}(k)$ and $M_{k} = \lambda_{max}(k)$ \cite{elenberg2018}. 

Thus, by Theorem~\ref{thm:param}, we can also write the bounds for $f_S(A)$ in terms of eigenvalues $\frac{\lambda_{min}(s)}{\lambda_{max}(t)} \tilde f_S(A) \leq f_S(A) \leq \frac{\lambda_{max}(s)}{\lambda_{min}(t)} \tilde f_S(A)$, where $\tilde f_S(A) =\sum_{a\in A} f_S(a)$. With $g_S(A) = \frac{\lambda_{min}(s)}{\lambda_{max}(t)} \tilde f_S(A)$ and $h_S(A) = \frac{\lambda_{max}(s)}{\lambda_{min}(t)} \tilde f_S(A)$,  we get that that the objective is a $(\frac{\lambda_{min}(t)}{ \lambda_{max}(t)})^2$-differentially submodular function. Since $2k\geq t$, we get the desired result.

In the case where there is a diversity regularization term in the objective $f_{\texttt{div}}(S) = \ell_{\texttt{reg}}(\bw^{(S)}) + d(S)$, we have   $$\frac{\lambda_{min}(s)}{\lambda_{max}(t)} \tilde f_S(A) + d_S(A) \leq (f_{\texttt{div}})_S(A) \leq \frac{\lambda_{max}(s)}{\lambda_{min}(t)} \tilde f_S(A) + d_S(A).$$ With  $g_S(A) = \frac{\lambda_{min}(s)}{\lambda_{max}(t)} \tilde f_S(A) + d_S(A)$ and $h_S(A) = \frac{\lambda_{max}(s)}{\lambda_{min}(t)} \tilde f_S(A) + d_S(A)$, we  get that $g_S(A) / h_S(A)  \geq \frac{\lambda_{min}(s) \lambda_{min}(t)}{\lambda_{max}(s) \lambda_{max}(t)} \geq (\frac{ \lambda_{min}(t)}{\lambda_{max}(t)})^2$ since $d_S(A) \geq 0$. Since $2k\geq t$, this concludes the proof.
\end{proof}
\begin{remark}
Since $\lambda_{max}(s) = 1$, the upper bound of $\frac{\lambda_{max}(s)}{\lambda_{min}(t)} \tilde f_S(A) \leq f_S(A)$ is consistent with the result in Lemma 2.4 from Das and Kempe \cite{das2011} that shows that the weak submodularity ratio can be lower bounded by $\lambda_{min}$.
\end{remark}
\subsection{Proof of Corollary \ref{corollary:log}}
\begin{proof}
The first portion of the proof relies on the result from Elenberg et al. \cite{elenberg2018}. In general, log-likelihood functions of generalized linear models (GLMs) are not RSC/RSM, but their result shows that log-likelihood objectives are
RSC/RSM with parameters $m$ and $M$ under mild conditions of the feature matrix. 

Our result follows directly from Theorem \ref{thm:param}. The case where there is a diversity regularization term then follows similarly as for Corollary~\ref{corollary:lin}.
\end{proof}
\subsection{Proof of Corollary \ref{corollary:bayes}}
\begin{proof}
In the case where there is no diversity regularization term, we can upper bound the submodularity ratio to prove differential submodularity.

We first lower bound the marginal contribution of a set $A$ to $S$, $(f_{\texttt{A-opt}})_S(A)$ and then upper bound the marginal contribution of one element $a$ to the set $S$, $(f_{\texttt{A-opt}})_S(a)$.
\begin{eqnarray}\label{bayes:lower}
f_S(A) &=& \sum_{i=1}^d \frac{1}{\beta^2 + \sigma^{-2}\sigma_i^2 (\bX_S)} - \sum_{j=1}^d\frac{1}{\beta^2 + \sigma^{-2}\sigma_i^2 (\bX_{S\cup A})} \nonumber \\
&=& \sum_{i=1}^d \frac{\sigma^{-2} [\sigma_i^2 (\bX_{S\cup A}) - \sigma_i^2 (\bX_S) ]}{(\beta^2 + \sigma^{-2}\sigma_i^2 (\bX_S)) (\beta^2 + \sigma^{-2}\sigma_i^2 (\bX_{S\cup A}))} \nonumber \\
&\geq& (\beta^2 + \sigma^{-2} \sigma^2_{max}(\bX))^{-2} \sum_{i=1}^d \sigma^{-2}[\sigma_i^2(\bX_{S\cup A}) - \sigma_i^2(\bX_S)] \nonumber \\
&=& (\beta^2 + \sigma^{-2} \| \bX \|^2)^{-2} \sum_{i=1}^d \sigma^{-2} [\lambda_i (\bX_{S\cup A} \bX^T_{S\cup A}) - \lambda_i (\bX_{S} \bX^T_{S})] \nonumber \\
&=& (\beta^2 + \sigma^{-2} \| \bX \|^2)^{-2} \sigma^{-2} [\Tr(\bX_{S\cup A} \bX^T_{S\cup A}) - \Tr (\bX_{S} \bX^T_{S})] \nonumber \\
&=& (\beta^2 + \sigma^{-2} \| \bX \|^2)^{-2} \sigma^{-2} [\Tr(\bX_{S} \bX^T_{S} + \bX_{A} \bX^T_{A}) - \Tr (\bX_{S} \bX^T_{S})] \nonumber \\
&=& (\beta^2 + \sigma^{-2} \| \bX \|^2)^{-2}  \sigma^{-2} \Tr(\bX_{A} \bX^T_{A})  \nonumber \\
&=& (\beta^2 + \sigma^{-2} \| \bX \|^2)^{-2} \sum_{a\in A} \sigma^{-2} \Tr(\bx_{a} \bx^T_{a})  \nonumber \\
&=& (\beta^2 + \sigma^{-2} \| \bX \|^2)^{-2} \sum_{a\in A}   \| \bx_a\|^2  \nonumber \\
&=& \sigma^{-2}(\beta^2 + \sigma^{-2} \| \bX \|^2)^{-2} |A|
\end{eqnarray}

\begin{eqnarray}\label{bayes:upper}
\sum_{a\in A} f_S(a) &=& \sum_{a\in A}  \sum_{i=1}^d \frac{1}{\beta^2 + \sigma^{-2}\sigma_i^2 (\bX_S)} - \sum_{j=1}^d\frac{1}{\beta^2 + \sigma^{-2}\sigma_i^2 (\bX_{S\cup a})} \nonumber \\
&\leq& \sum_{a\in A} \frac{1}{\beta^2 + \sigma^{-2}\sigma^2_d(\bX_S)} - \frac{1}{\beta^2 + \sigma^{-2}\sigma_1^2 (\bX_{S\cup a})} \nonumber \\
&\leq& \sum_{a\in A} \frac{1}{\beta^2} - \frac{1}{\beta^2 + \sigma^{-2}\sigma_1^2 (\bX_{S\cup a})} \nonumber \\
&=&  \sum_{a\in A} \frac{\sigma^{-2}\sigma_1^2 (\bX_{S\cup a})}{\beta^2(\beta^2 + \sigma^{-2}\sigma_1^2 (\bX_{S\cup a}))} \nonumber \\
&=&  |A|\frac{\sigma^{-2}\| \bX\|^2}{\beta^2(\beta^2 + \sigma^{-2}\| \bX\|^2)} 
\end{eqnarray}

Combining (\ref{bayes:lower}) and (\ref{bayes:upper}), yields
$$ \frac{\sum_{a\in A} f_S(a) }{f_S(A)} \leq \frac{|A|\frac{\sigma^{-2}\| \bX\|^2}{\beta^2(\beta^2 + \sigma^{-2}\| \bX\|^2)}  } {\sigma^{-2}(\beta^2 + \sigma^{-2} \| \bX \|^2)^{-2} |A|} = \frac{ \| \bX \|^2(\beta^2+\sigma^{-2} \|\bX\|^2) } {\beta^{2}}.$$

Bian et al. \cite{bian2017} showed that the submodularity ratio can be lower bounded by $\frac{\beta^2}{\|\bX\|^2(\beta^2+\sigma^{-2} \|\bX\|^2)}$. 

 With $g_S(A) =\frac{\beta^2}{\|\bX\|^2(\beta^2+\sigma^{-2} \|\bX\|^2)} \tilde f_S(A)$ and $h_S(A) = \frac{\|
 \bX\|^2 (\beta^2 + \sigma^{-2} \| \bX \|^2)}{\beta^2} \tilde f_S(A)$,  we get that that the objective is a $\gamma^2$-differentially submodular function where $\gamma = \frac{\beta^2}{\|\bX\|^2(\beta^2+\sigma^{-2} \|\bX\|^2)}$. 

In the case where there is a diversity regularization term in the objective, we can follow similar reasoning from Corollary \ref{corollary:lin} to conclude the proof.
\end{proof}
\section{Extension to $R^2$ Objective}\label{appendix:eig}
\subsection{Goodness of Fit}
We introduce the formal definition of the $R^2$ objective function, which is widely used to measure goodness of fit in statistical applications.
\begin{definition}
\cite{johnson2004}
Let $S \subseteq N$ be a set of variables $\bX_S$ and a linear predictor $\hat \by = \sum_{i\in S}\beta_i\bX_i$ of $\by$, the squared multiple correlation is defined as
\begin{eqnarray}
R^2(S) = \frac{\text{Var}(\by) - \E[(\by - \hat \by)^2]}{\text{Var}(\by)} \nonumber
\end{eqnarray}
where $\beta_i = (\bC_S)^{-1} \bb_S$ for $i \in S$.

\end{definition}
We assume that the predictor random variables are normalized to have mean 0 and variance 1, so we can simplify the definition above to $R^2(S) = 1 - \E [(\by-\hat \by)^2]$. Thus, we can rephrase the definition as $R^2(S) = \bb_S^T (\bC_S)^{-1} \bb_S$. \cite{johnson2004}.
\subsection{Feature Selection}
{\bf Objective.} \ For a response variable $\by \in \rrr^d$, the objective is the maximization of the $R^2$ goodness of fit for $\by$ given the feature set $S$: 
\begin{eqnarray}
f(S) = R^2(S) = \bb_S^T (\bC_S)^{-1} \bb_S \nonumber
\end{eqnarray}
where $\bb$ corresponds to the covariance between $\by$ and the predictors.

To define the marginal contribution of a set $A$ to the set $S$ of the $R^2$ objective function, we can write $R^2_S(A) = 
(\bb^S_A)^T(\bC^S_A)^{-1}\bb^S_A$, where $\bb^S$ is the covariance vector corresponding to the residuals of $i \in A$ to $S$, i.e. $\{ \text{Res} (\bx_1, \bX_S), \text{Res} (\bx_2, \bX_S), \ldots, \text{Res} (\bx_n, \bX_S)\}$ and $\bC^S_A$ is the covariance matrix corresponding to the residuals. The marginal contribution of an element is $R^2_S(a) = (\bb^S_a)^T \bb^S_a$.

\begin{lemma}\label{lemma:eig}
The feature selection objective defined by $f(S) = R^2(S)$ is a $\frac{\lambda_{min}(\bC^S_A)}{\lambda_{max}(\bC^S_A)}$-differentially submodular function such that for all $S,A\subseteq N$,
$$ g_S(A) = \frac{1}{\lambda_{max}(\bC^S_A)} \tilde f_S(A) \leq f_S(A) \leq \frac{1}{\lambda_{min}(\bC^S_A)} \tilde f_S(A) = h_S(A),$$
where $\tilde f_S(A) = \sum_{a\in A} f_S(a)$.
\end{lemma}
\begin{proof}
The marginal contribution of set $A$ to set $S$ of the feature selection objective function is defined as $R^2_S(A) = 
(\bb^S_A)^T(\bC^S_A)^{-1}\bb^S_A$. Because we know that $(\bC^S_A)^{-1}$ is a symmetric matrix, we can upper and lower bound the marginals using the eigenvalues of $(\bC^S_A)^{-1}$.
\begin{eqnarray}
\frac{1}{\lambda_{max}(\bC^S_A)}  \sum_{a\in A}f_S(a) &=&  \frac{1}{\lambda_{max}(\bC^S_A)} (\bb^S_A)^T \bb^S_A  \nonumber \\
&=& \lambda_{min}((\bC^S_A)^{-1}) (\bb^S_A)^T \bb^S_A \nonumber \\
&\leq& (\bb^S_A)^T(\bC^S_A)^{-1}\bb^S_A \nonumber \\ 
&=& f_S(A) \nonumber \\
&\leq& \lambda_{max}((\bC^S_A)^{-1}) (\bb^S_A)^T \bb^S_A \nonumber \\
&\leq& \frac{1}{\lambda_{min}(\bC^S_A)} (\bb^S_A)^T \bb^S_A \nonumber \\
&=& \frac{1}{\lambda_{min}(\bC^S_A)}  \sum_{a\in A}f_S(a) \nonumber 
\end{eqnarray}
By letting $ \tilde f_S(A) = \sum_{a\in A} f_S(a)$, we complete the proof and show that the marginals can be bounded by modular functions.
\end{proof}

\begin{remark}
This is a more general form of Lemma 3.3 from Das and Kempe \cite{das2011}. Our result is on the marginals of $f$ and reduces to their result for $S=\emptyset$.
\end{remark}

\begin{remark}
If $\lambda_{min} = \lambda_{max}$, the matrix has one eigenvalue of multiplicity greater than 1 and the covariance matrix is a multiple of the identity matrix. This implies the set of predictors is uncorrelated and that the objective function for feature selection is submodular. Otherwise, we have $\alpha = \frac{\lambda_{min}}{\lambda_{max}} < 1$.
\end{remark}

\section{Additional Algorithm Detail} \label{appendix:queries}
We briefly discuss how to estimate the expectations that appear in the algorithm. We also discuss how to estimate $\texttt{OPT}$ and differential submodularity parameter $\alpha$. For the full algorithm and details, see Appendix A.C.2 in \cite{balkanski2018}.

Since we do not know the value of $\mathbb E_{R\sim \mathcal U(X)} [f_S(R)] $, we can estimate it with $m$ samples. We first randomly select sets uniformly $R_1, R_2, ...R_m \sim \mathcal U (X)$ and compute $f_S(R_i)$. Then we can average these calculations to estimate the expected marginal contribution. Balkanski et al. discuss the number of samples needed to bound the error of these estimates \cite{balkanski2018}. Specifically, with $m=\frac{1}{2} (\frac{\texttt{OPT}}{\epsilon})^2 \log (\frac{2}{\delta})$, then with probability at least $1-\delta$, $$\left|\left(\frac{1}{m} \sum_{i=1}^m f(S \cup R_i) - f(S)\right) - \mathbb E_{R\sim \mathcal U(X)} [f_S(R)] \right| \leq \epsilon$$. 

Similarly, let $m=\frac{1}{2} (\frac{\texttt{OPT}}{\epsilon})^2 \log (\frac{2}{\delta})$, then for all $S\subseteq N$ and $a\in N$, with probability at least $1-\delta$ over samples $R_1, ...,R_m$,
$$\left|\left(\frac{1}{m} \sum_{i=1}^m f(S \cup R_i \cup \{a\}) - f(S\cup R_i \backslash \{a\})\right) - \mathbb E_{R\sim \mathcal U(X)} [f_{S\cup R \backslash \{a\}}(a)] \right| \leq \epsilon.$$

Thus, for $m = n (\frac{\texttt{OPT}}{\epsilon})^2 \log (\frac{2n}{\delta})$ total samples in one round, we can get $\epsilon$-estimates for marginal contributions. For proof details, see Lemma 6 in \cite{balkanski2018}. We note that in practice, we observe comparable terminal values compared to the greedy algorithm even with much fewer number of samples.

To estimate \texttt{OPT}, we can ``guess'' the value of \texttt{OPT} and run several of these guesses in parallel. One can set $\texttt{OPT} \in \{(1+\epsilon)^i \max_{a\in N} f(a) : i \in \left[\frac{\ln(n)}{\epsilon} \right] \}$. One such value $i$ is guaranteed to be a $(1-\epsilon)$-approximation to \texttt{OPT} \cite{balkanski2018}. Similarly for the differential submodularity parameter $\alpha$, we can guess values so that $\alpha \in \{(1+\epsilon)^i : i \in \left[\frac{\ln(n)}{\epsilon} \right] \}$ and run these guesses in parallel. In practice, we found that the algorithm performance was not very sensitive to parameter estimates and we could observe comparable terminal value without much parameter tuning.

\section{Proof of Theorem \ref{thm:apx} for \textsc{Dash}}\label{appendix:main}

We first prove several lemmas before proving the theorem. 

\subsection{Proofs of Lemmas Leading to Theorem \ref{thm:apx}} \label{appendix:lemma}
We first begin by proving the following lemma to bound the marginal contribution of the optimal set to the solution set.

\begin{lemma} \label{lemma:bound}
 Let $R_i \sim \mathcal U(X)$ be the random set at iteration $i$ of \textsc{Dash}$(N,S,r,\delta)$. 
For all $S \subseteq N$ and $r, \rho > 0$, if the algorithm has not terminated after $\rho$ iterations, then
\begin{eqnarray}
\mathbb E_{R_i} [f_{S\cup (\cup_{i=1}^\rho R_i)} (O)] \geq (1- \frac{\rho}{r}) (f(O) - f(S))
\end{eqnarray}
\end{lemma}
Using Lemma \ref{lemma:bound}, we can complete the proof for Lemma \ref{lemma:apx}.
\begin{proof}

\begin{align*}
\mathbb E_{R_i} [f_{S\cup (\cup_{i=1}^\rho R_i)} (O)] &= \mathbb E_{R_i} [f_{S} (O \cup (\cup_{i=1}^\rho R_i))] 
- \mathbb E_{R_i}[f_S(\cup_{i=1}^\rho R_i)]   \nonumber \\
&\geq f(O) - f(S) - \frac{1}{\alpha} \sum_{i=1}^\rho \mathbb E_{R_i} [f_S(R_i)] \nonumber \\
&\geq f(O) - f(S) - \frac{1}{\alpha} \alpha\rho(\frac{1-\epsilon}{r} (f(O) - f(S)) \nonumber \\
&\geq (1- \frac{\rho}{r}) (f(O) - f(S)) && \nonumber 
\end{align*}
where the first inequality follows from monotonicity and differential submodularity and the second inequality follows from the \texttt{while} loop in \textsc{Dash}.
\end{proof}

\begin{lemma} \label{lemma:apx}
For each iteration of \textsc{Dash} and for all $S \subseteq N$ and $\epsilon > 0$, if $r\geq 20\rho\epsilon^{-1}$ then the marginal contribution of the elements of $X_\rho$ that survive $\rho$ iterations satisfy
$$f_S(X_\rho) \geq \frac{\alpha^2}{r}(1-\epsilon) (f(O) - f(S)) $$
\end{lemma}

\begin{proof}
We want to show a bound on the marginal contribution of the elements that survive $\rho$ iterations of the algorithm. To prevent the propagation of the $\alpha$ factor, we upper and lower bound $f$ by two submodular functions $h$ and $g$ for our analysis, excluding the queries made by the algorithm.

Let $O = \{ o_1, \ldots , o_k\}$ be the optimal solutions of $f$ and $O_l = \{o_1, \ldots, o_l\}$ be a subset of the optimal elements in some
arbitrary order. However, we define the thresholds in terms of submodular function $h$. Then we define
\begin{eqnarray}
\Delta_l &:=& \mathbb E_{R_i}[h_{S\cup O_{l-1} \cup (\cup_{i=1}^\rho R_i \backslash \{o_l\}} (o_l)] \label{eqn:deltal} \\
\Delta &:=& \frac{1}{k} \mathbb E_{R_i} [h_{S\cup (\cup_{i=1}^\rho R_i)} (O)]  \label{eqn:delta} 
\end{eqnarray}
Let $r \geq \frac{20\rho}{\epsilon}$. Let $T$ be the set of elements surviving $\rho$ iterations in $O$, $T \subseteq X_\rho $,
$T \subseteq O $, where
\begin{eqnarray}
T = \{ o_l | \Delta_l \geq (1-\frac{\epsilon}{4}) \Delta \} \label{eqn:T} 
\end{eqnarray}

For $o_l \in T$ and using differential submodularity properties, 
\begin{align}
\mathbb E_{R_i}[f_{S\cup (\cup_{i=1}^\rho R_i \backslash \{o_l\}) }(o_l)]  &\geq \mathbb E_{R_i}[g_{S\cup (\cup_{i=1}^\rho R_i \backslash \{o_l\}) }(o_l)] \nonumber \\
 &\geq \mathbb E_{R_i}[\alpha h_{S\cup (\cup_{i=1}^\rho R_i \backslash \{o_l\}) }(o_l)] \nonumber \\
&\geq \alpha\mathbb E_{R_i}[h_{S\cup O_{l-1} \cup (\cup_{i=1}^\rho R_i \backslash \{o_l\}} (o_l)] \nonumber \\
&\geq\alpha(1-\frac{\epsilon}{4}) \Delta && \text{Definition (\ref{eqn:T})} \nonumber \\
&\geq \frac{\alpha}{k}(1-\frac{\epsilon}{4}) \mathbb E_{R_i} [f_{S\cup (\cup_{i=1}^\rho R_i)} (O)]   \nonumber \\
&\geq \frac{\alpha}{k} (1-\frac{\epsilon}{4})(1- \frac{\rho}{r}) (f(O) - f(S)) && \text{Lemma \ref{lemma:bound}} \nonumber \\
&\geq \frac{\alpha}{k}(1+\frac{\epsilon}{2})(1-\epsilon)(f(O) - f(S))  &&  \text{$r\geq \frac{20\rho}{\epsilon}$}\label{eqn:deltabound}
\end{align}
which shows that elements in $T$ survive the elimination process (as they are not filtered out from set $X$ in the
algorithm definition).

Now we complete the proof by showing $f_S(T)$ is bounded by $\frac{\alpha^2}{r}(1-\epsilon)(f(O) - f(S))$ which 
effectively terminates the algorithm.

Similar to the result in Lemma 2 of Balkanski et al. \cite{balkanski2019}, from properties of submodularity of $g$ and $h$, we have
\begin{eqnarray}
\sum_{o_l\in T}\Delta_l \geq k \frac{\epsilon}{4}\Delta \label{eqn:Del}
\end{eqnarray}

By submodularity,
\begin{align}
f_S(T) &\geq g_S(T) \nonumber \\
&\geq \alpha h_S(T) \nonumber \\
&\geq \alpha \sum_{o_l \in T} h_{S\cup O_{l-1}}(o_l)  \nonumber \\
&\geq \alpha \sum_{o_l\in T} \mathbb E[h_{S\cup O_{l-1} \cup (\cup_{i=1}^\rho R_i \backslash \{o_l\}} (o_l)]  \nonumber \\
&= \alpha \sum_{o_l\in T} \Delta_l && \text{Definition (\ref{eqn:deltal})} \nonumber \\
&\geq (1-\delta)k\Delta \frac{\epsilon}{4} &&\text{from (\ref{eqn:Del})} \nonumber \\
&\geq \alpha \frac{\epsilon}{4} \mathbb E_{R_i} [f_{S\cup (\cup_{i=1}^\rho R_i)} (O)]  \nonumber
\end{align}
where the second and third inequalities follow from properties of submodularity.
Finally,
\begin{align}
f_S(X_\rho) &\geq f_S(T) && \text{monotonicity} \nonumber \\
&=\alpha( \frac{\epsilon}{4}) (1- \frac{\rho}{r}) (f(O) - f(S)) && \text{Definition (\ref{eqn:delta})}  \nonumber \\
&\geq \frac{\alpha^2}{r}(1-\epsilon) (f(O) - f(S)) &&  \text{$r\geq \frac{20\rho}{\epsilon}$}  \nonumber
\end{align}
\end{proof}

We now present a lemma for the termination of the algorithm in $\mathcal O (\log n)$ rounds.
\begin{lemma}\label{lemma:log}
Let $X_i$ and $X_{i+1}$ be the sets of surviving elements at the start and end of iteration $i$ of the \texttt{while} loop of \textsc{Dash}. 
For all $S\subseteq N$ and $r,i,\epsilon >0$, if the algorithm does not terminate at iteration $i$, then
$$|X_{i+1}| < \frac{|X_i|}{1 + \epsilon/2}$$
\end{lemma}

\begin{proof}
We consider $R_i \cap X_{i+1}$ to bound the number of surviving elements in $X_{i+1}$. To prevent the propagation of the $\alpha$ factor, we can bound the function $f$ by its submodular bounds.
\begin{align}
\mathbb E[f_S(R_i \cap X_{i+1})]  &\geq  \mathbb E[g_S(R_i \cap X_{i+1})] \nonumber \\
&\geq \alpha \mathbb E[\sum_{a\in R_i \cap X_{i+1}} h_{S\cup (R_i \cap X_{i+1} \backslash a)} (a) ] \nonumber \\
&\geq \alpha \mathbb E[\sum_{a\in X_{i+1}} \mathbbm{1}_{a\in R_i} \cdot h_{S\cup (R_i  \backslash a)} (a) ]  \nonumber \\
&= \alpha \sum_{a\in X_{i+1}} \mathbb E[\mathbbm{1}_{a\in R_i} \cdot h_{S\cup (R_i  \backslash a)} (a) ]\nonumber \\
&= \alpha \sum_{a\in X_{i+1}} \mathbb P[a\in R_i] \cdot \mathbb E [h_{S\cup (R_i  \backslash a)} (a) | a\in R_i]  \nonumber \\
&\geq \alpha \sum_{a\in X_{i+1}} \mathbb P[a\in R_i] \cdot \mathbb E [h_{S\cup (R_i  \backslash a)} (a) ]  \nonumber \\
&\geq \alpha \sum_{a\in X_{i+1}} \mathbb P[a\in R_i] \cdot \mathbb E [f_{S\cup (R_i  \backslash a)} (a) ]  \nonumber \\
&\geq \alpha\sum_{a\in X_{i+1}} \mathbb P[a\in R_i] \cdot \frac{\alpha}{k}(1+\epsilon/2)(1-\epsilon)(f(O) - f(S)) \nonumber \\
&= \alpha^2\frac{|X_{i+1}|}{|X_i|}\frac{k}{r} \cdot \frac{1}{k}(1+\epsilon/2)(1-\epsilon)(f(O) - f(S)) \nonumber \\
&= \alpha^2\frac{|X_{i+1}|}{r|X_i|}(1+\epsilon/2)(1-\epsilon)(f(O) - f(S))  \label{eqn:ineq1}
\end{align}
where the first and fourth inequalities are due to differential submodularity.

Since the elements are discarded from the \texttt{while} loop of the algorithm, we can bound $\mathbb E [f_S(R_i \cap X_{i+1})]$ using monotonicity so that 
\begin{eqnarray}
\mathbb E [f_S(R_i \cap X_{i+1})] \leq \mathbb E [f_S(R_i)] < \alpha^2(1-\epsilon) (f(O) - f(S))/r. \label{eqn:ineq2}
\end{eqnarray}

Combining (\ref{eqn:ineq1}) and (\ref{eqn:ineq2}) yields
$$(1-\epsilon) (f(O) - f(S))/r  > \frac{1}{\alpha^2}\mathbb E [f_S(R_i)] \geq \frac{|X_{i+1}|}{r|X_i|}(1+\epsilon/2)(1-\epsilon)(f(O) - f(S))$$

We can conclude that $|X_{i+1}| < |X_i|/(1+\epsilon/2)$ by simplifying that above inequality.
\end{proof}
\begin{lemma} \label{lemma:log}
For all $S \subseteq N$, if $r \geq 20\epsilon^{-1}\log_{(1+\epsilon/2)}(n)$ then \textsc{Dash}$(N,S,r,\delta)$ 
terminates after at most $\mathcal O (\log n)$ rounds.
\end{lemma}
\begin{proof}
If the algorithm has not terminated after $\log_{1+\epsilon/2}(n)$ rounds, then, by Lemma \ref{lemma:log}, at most 
$k/r$ elements survived $\rho = \log_{1+\epsilon/2}(n)$ iterations. By Lemma \ref{lemma:apx}, the set of surviving elements satisfies 
$ f_S(X_\rho)  \geq \frac{\alpha^2}{r}(1-\epsilon)(f(O) - f(S))$. Since there are only $k/r$ surviving elements, $R=X_\rho$ and 
$$ f_S(R) = f_S(X_\rho)  \geq \frac{\alpha^2}{r}(1-\epsilon)(f(O) - f(S))$$
\end{proof}
\subsection{Proof of Theorem \ref{thm:apx}}

\begin{proof}
We prove the theorem by induction. From Lemma \ref{lemma:log}, we know $$f(S_i) \geq f(S_{i-1}) + \alpha^2 \frac{1-\epsilon}{r}(f(O) -f(S_{i-1}))$$

By subtracting $f(O)$, this is equivalent to $$f(S_{i}) - f(O) \geq (1- \alpha^2 \frac{1-\epsilon}{r}) [f(S_{i-1}) - f(O)]$$
By induction and rearranging, we have
\begin{eqnarray}
f(S_i) -f(O) &\geq& (1- \alpha^2 \frac{1-\epsilon}{r})^i (-f(O))\nonumber \\
&=& -(1- \alpha^2 \frac{1-\epsilon}{r})^i f(O)  \nonumber  
\end{eqnarray}
By setting $i=r$ and rearranging, we have 
\begin{eqnarray}
f(S) &\geq (1 - (1- \alpha^2 \frac{1-\epsilon}{r})^{r}) f(O) \nonumber \\
&\geq (1-e^{-\alpha^2 (1-\epsilon)}) f(O)\nonumber \\
&\geq (1-1/e^{\alpha^2}  -\alpha^2 \epsilon) f(O) \nonumber
\end{eqnarray}
\end{proof}

\section{Additional Detail for Experiments}
\subsection{Experimental Setup}\label{app:exp_details}
All algorithms were implemented in Python 3.6. Experiments on third-party datasets were conducted on AWS EC2 C4 with 2.9 GHz Intel Xeon E5-2666 v3 Processors on 16 or 36 cores. Experiments on synthetic datasets ran on 3.1 GHz Intel Core i7 processors on 8 cores. 
\subsection{Datasets} \label{app:data}
This section details the generation of synthetic data and the real clinical and biological datasets we used for experiments. D1 and D2 are used in linear regression and Bayesian experimental design applications and D3 and D4 are used in logistic regression for classification applications.
\begin{itemize}
\item
{\bf D1: Synthetic Dataset for Regression and Experimental Design. \ } 
We generated $500$ features by sampling from a multivariate normal distribution. Each feature is normalized to have mean 0 and variance 1. Furthermore, features have a covariance of 0.4 to guarantee differential submodularity. To generate our response variable $\by$, we sample the coefficient $\beta \sim \mathcal U (-2, 2)$ for a subset of size 100 from the feature set and compute $\by$ after adding a small noise term to the coefficients. Our goal is to select features that have coefficients of large magnitude and accurately predict the response variable $\by$.

We generated the dataset for experimental design similarly. We generated $256$ features and 1024 samples by sampling from a multivariate normal distribution. Each feature is normalized to have mean 0 and variance 1. Features have a covariance of 0.8. Each row is then normalized to have $\ell_2$ norm of 1;

\item {\bf D2: Clinical Dataset for Regression and Experimental Design. \ } We used a publicly available dataset with 53,500 samples from 74 patients with 385 features and want to select a smaller set of features that can accurately predict the location on the axial axis from an image of the brain slice. For experimental design, we sample 1000 rows from the dataset to comprise our sample space and normalize rows to have $\ell_2$ norm of 1;
\item {\bf D3: Synthetic Dataset for Classification. \ } We generated a synthetic dataset for logistic regression using a similar methodology as the synthetic regression dataset. We select a set of 50 true support features from a set of 200 and generate the coefficients using $\mathcal U (-2, 2)$. However, instead of a numerical response variable, we create a two-class classification problem by transforming the continuous $\by$ into probabilities and assigning the class label using a threshold of 0.5. The goal is to select features to perform binary classification on the synthetic dataset by using the log likelihood objective;

\item  {\bf D4: Biological Dataset for Classification. \ } We used clinical data that contains the presence or absence of 2,500 genes in 10,633 samples from various patients. In this 5-class multi-classification problem, we want to select a small set of genes that can accurately predict the site of cancer metastasis (spleen, colon, parietal peritoneum, mesenteric lymph node, and intestine). 
\end{itemize}

\subsection{Benchmarks} \label{app:bench}
We compared \apx to these algorithms: 
\begin{itemize}

\item{{\bf $\textsc{Random}$.} \ In one round, this algorithm randomly selects $k$ features to create the solution set;}

\item{{\bf $\textsc{Top-$k$}$.} \ In one round, this algorithm selects the $k$ features whose individual objective value is largest;}

\item{{\bf $\textsc{SDS}_{\textsc{MA}}$.} \ This uses the traditional greedy algorithm to select elements with the largest marginal contribution at each round \cite{krause2010}. In each round, the algorithm adds one element to the solution set; }

\item{{\bf Parallel $\textsc{SDS}_{\textsc{MA}}$.} \ To compare parallel runtime between \apx and greedy, we also implemented a parallelized version of the $\textsc{SDS}_{\textsc{MA}}$ algorithm. In each round, the algorithm computes the marginal contribution of each element to the intermediate solution set. These oracle queries are parallelized across multiple cores. This is especially effective in settings where the oracle queries are computationally intensive;} %
\item{{\bf $\textsc{Lasso}$.}} This popular algorithm fits either a linear or logistic regression with an $\ell_1$ regularization term $\gamma$.  It is known that for any given instance that is $k$-sparse there exists a regularizer $\gamma_k$ that can recover the $k$ sparse features.  Using \textsc{Lasso} to find a fixed set of features is computationally intensive since in general, finding the regularizer is computationally intractable~\cite{MY12} and even under smoothed analysis its complexity is at least \emph{linear} in the dimension of the problem~\cite{singer2018}.  We therefore used sets of values returned by \textsc{Lasso} for varying choices of regularizers and use these values to benchmark the objective values returned by \textsc{Dash} and the other benchmarks.  
\end{itemize}

\section{Observation on Worst Case Bound}\label{appendix:worse}
In the special case of feature selection where there is no diversity term, we can get an improved approximation guarantee of $\gamma^2$, where $\gamma = m/M$.

We can bound the objective function $f(S)$ by the modular function $\sum_{a\in S} f(a)$ so that $ \frac{m}{M} \sum_{a\in S} f(a) \leq f(S) \leq \frac{M}{m} \sum_{a\in S} f(a)$.
Then, for the \textsc{Top-k} algorithm, where we select the best $k$ elements by their value $f(a)$, we get the following approximation guarantee.

$$ f(S) \geq \frac{m}{M} \sum_{a\in S}f(a) \geq \frac{m}{M} \sum_{o\in O}f(o) \geq (\frac{m}{M})^2 f(O) = \gamma^2 f(O)$$

where the first and last inequalites come from differential submodularity properties and the second inequality follows from selecting the best $k$ elements.
\begin{remark}
In the case where $\gamma=1$, $f(S)$ is a submodular function. In the context of feature selection, when $\gamma=1$, the features are linearly independent and one can obtain the optimal solution by selecting the $k$ features that have the largest marginal contributions to the empty set.
\end{remark}

\end{document}